\csname keepappendix\endcsname

\documentclass{article}
\PassOptionsToPackage{sort&compress}{natbib}
\RequirePackage{amsmath,amsthm,enumitem}
\usepackage[accepted]{icml2023}

\usepackage{float}
\usepackage{rotating}
\usepackage{microtype}
\usepackage{graphicx}
\usepackage{subfigure}
\usepackage[utf8]{inputenc} %
\usepackage[T1]{fontenc}    %
\usepackage{url}            %
\usepackage{booktabs}       %
\usepackage{amsfonts}       %
\usepackage{amsmath}
\usepackage{mathtools}
\usepackage{color, colortbl}
\usepackage{caption}
\usepackage{soul}
\usepackage{overpic}

\newtheorem{theorem}{Theorem}

\newtheorem{corollary}{Corollary}
\newtheorem{proposition}{Proposition}
\newtheorem{lemma}{Lemma}
\newtheorem*{assumption*}{Assumption}
\newtheorem*{metatheorem*}{Metatheorem}

\definecolor{Gray}{gray}{0.9}

\newcommand{\dd}{\mathrm{d}}
\newcommand{\tr}{\mathrm{tr}}
\newcommand{\eulermac}{\gamma_{\text{EM}}}
\newcommand{\GP}{\mathcal{GP}}
\newcommand{\cov}{\mathrm{Cov}}

\newcommand{\new}[1]{#1}

\icmltitlerunning{Monotonicity and Double Descent in GPs}

\title{Monotonicity and Double Descent in Uncertainty Estimation with Gaussian Processes}

\ifdefined\techreport
\author{
Liam Hodgkinson%
\thanks{School of Mathematics and Statistics, University of Melbourne, Australia. Email: \texttt{lhodgkinson@unimelb.edu.au}}
\and 
Chris van der Heide%
\thanks{Department of Electrical and Electronic Engineering, University of Melbourne, Australia. Email:  \tt{chris.vdh@gmail.com}}
\and 
Fred Roosta%
\thanks{School of Mathematics and Physics, University of Queensland, Australia; and International Computer Science Institute, Berkeley, CA, USA. Email:  \tt{fred.roosta@uq.edu.au}}
\and 
Michael W. Mahoney%
\thanks{Department of Statistics, University of California at Berkeley, Berkeley, CA, USA; Lawrence Berkeley National Laboratory, Berkeley, CA, USA; and International Computer Science Institute, Berkeley, CA, USA. Email: \texttt{mmahoney@stat.berkeley.edu}}
}
\else
\author{Antiquus S.~Hippocampus, Natalia Cerebro \& Amelie P. Amygdale \thanks{ Use footnote for providing further information
about author (webpage, alternative address)---\emph{not} for acknowledging
funding agencies.  Funding acknowledgements go at the end of the paper.} \\
Department of Computer Science\\
Cranberry-Lemon University\\
Pittsburgh, PA 15213, USA \\
\texttt{\{hippo,brain,jen\}@cs.cranberry-lemon.edu} \\
\And
Ji Q. Ren \& Yevgeny LeNet \\
Department of Computational Neuroscience \\
University of the Witwatersrand \\
Joburg, South Africa \\
\texttt{\{robot,net\}@wits.ac.za} \\
\AND
Coauthor \\
Affiliation \\
Address \\
\texttt{email}
}
\fi

\begin{document}

\ifdefined\techreport
\maketitle
\else
\twocolumn[
\icmltitle{Monotonicity and Double Descent in \\ Uncertainty Estimation with Gaussian Processes}

\begin{icmlauthorlist}
\icmlauthor{Liam Hodgkinson}{melb}
\icmlauthor{Chris van der Heide}{melbeng}
\icmlauthor{Fred Roosta}{uq,cires,icsi}
\icmlauthor{Michael W. Mahoney}{icsi,lbnl,berk}
\end{icmlauthorlist}

\icmlaffiliation{melb}{School of Mathematics and Statistics, University of Melbourne, Australia}
\icmlaffiliation{melbeng}{Department of Electrical and Electronic Engineering, University of Melbourne, Australia}
\icmlaffiliation{uq}{School of Mathematics and Physics, University of Queensland, Australia}
\icmlaffiliation{cires}{ARC Training Centre for Information Resilience, University of Queensland, Australia}
\icmlaffiliation{icsi}{International Computer Science Institute}
\icmlaffiliation{lbnl}{Lawrence Berkeley National
Laboratory}
\icmlaffiliation{berk}{Department of Statistics, University of California at Berkeley}

\icmlcorrespondingauthor{Liam Hodgkinson}{lhodgkinson@unimelb.edu.au}

\icmlkeywords{Double Descent, Gaussian Processes, Marginal Likelihood}

\vskip 0.3in
]
\fi
\printAffiliationsAndNotice{} %
\begin{abstract}
Despite their importance for assessing reliability of predictions, uncertainty quantification (UQ) measures for machine learning models have only recently begun to be rigorously characterized. One prominent issue is the \emph{curse of dimensionality}: it is commonly believed that the marginal likelihood should be reminiscent of cross-validation metrics and that both should deteriorate with larger input dimensions. 
We prove that by tuning hyperparameters to maximize marginal likelihood (the empirical Bayes procedure), the performance, as measured by the marginal likelihood, \emph{improves monotonically} with the input dimension.
On the other hand, we prove that cross-validation metrics exhibit qualitatively different behavior that is characteristic of \emph{double descent}. 
Cold posteriors, which have recently attracted interest due to their improved performance in certain settings, appear to exacerbate these phenomena. 
We verify empirically that our results hold for real data, beyond our considered assumptions, and we explore consequences involving synthetic covariates.

\end{abstract}

\section{Introduction}

With the recent success of overparameterized and nonparametric models for many predictive tasks in machine learning (ML), the development of the corresponding uncertainty quantification (UQ) has unsurprisingly become a topic of significant interest. 
Na\"{i}ve approaches for forward propagation of error and other methods for inverse uncertainty problems typically apply Monte Carlo methods under a Bayesian framework \citep{zhang2021modern}. 
However, the large-scale nature of many problems of interest results in significant computational challenges. 
One of the most successful approaches for solving inverse uncertainty problems is the use of \emph{Gaussian processes} (GP) \citep{williams2006gaussian}.
This is now frequently used for many predictive tasks, including  time-series analysis \citep{roberts2013gaussian}, regression and classification \citep{williams1998bayesian,williams2006gaussian}. GPs are also valuable in deep learning theory due to their appearance in the infinite-width limits of Bayesian neural networks \cite{N96,JGH18}.

A prominent feature of modern ML tasks is their large number of attributes: for example, in computer vision and natural language tasks, input dimensions can easily scale into the tens of thousands. This is concerning in light of the prevailing theory that GP performance often deteriorates in higher input dimensions.
This \emph{curse of dimensionality} for GPs has been rigorously demonstrated through error estimates for the kernel estimator \citep{von2004distance,jin2021learning}, showing that test error for most kernels scales in the number of data points as $\mathcal{O}(n^{-\alpha/d})$ for some $\alpha > 0$, where $d$ is the input dimension. This is further supported by empirical evidence \citep{spigler2020asymptotic}. However, it is well-known that Bayesian methods can perform well in high dimensions \citep{de2021high}, even outperforming their low-dimensional counterparts when properly tuned \citep{wilson2020bayesian}. Developments in the \emph{double descent} literature have helped to close this theory-practice gap by demonstrating that different behavior occurs when $n$ and $d$ scale \emph{proportionally}, and performance may actually \emph{improve} with larger input dimensions \cite{liu2021kernel}. Fortunately, ML datasets often fall into this regime.

Although the theoretical understanding of the predictive capacity of high-dimensional ML models continues to advance rapidly, analogous theoretical treatments for measures of uncertainty have only recently begun to bear fruit  \citep{clarte2022study,clarte2022theoretical}. Several common measures of model quality which incorporate inverse uncertainty quantification are Bayesian in nature, the most prominent of which are the \emph{marginal likelihood} and various forms of \emph{cross-validation}. Marginal likelihood and posterior distributions are often intractable for arbitrary models (e.g., Bayesian neural networks \citep{goan2020bayesian}),
yet their explicit forms are well known for GPs \citep{williams2006gaussian}. It is generally believed that performance under the marginal likelihood \emph{should not improve} with the addition of spurious covariates \cite{lotfi2022bayesian}. 
The celebrated work of \citet{fong2020marginal} relating marginal likelihood to cross-validation error would suggest that the marginal likelihood should behave similarly to test error, yet earlier work in statistical physics \cite{bruce1994statistical} suggests otherwise. The situation is further complicated as hyperparameters are not often fixed in practice, but are tuned relative to data, in a process known as \emph{empirical Bayes}.

An adjacent phenomenon is the \emph{cold posterior effect} (CPE): Bayesian neural networks exhibit optimal performance when the posterior is \emph{tempered} \cite{wenzel2020good}. As this effect has been observed in GPs as well \cite{adlam2020cold}, we focus our attention onto choices of hyperparameters which induce tempered posteriors. While we only encounter CPE in a limited capacity, we find that the cold posterior setting exacerbates more interesting qualitative behavior. Our main results (see Theorem~\ref{thm:Main} and Proposition~\ref{prop:DD}) are summarized as follows.
\emph{\begin{itemize}[leftmargin=*]
    \item \textbf{Monotonicity}: For two optimally regularized scalar GPs, both fit to a sufficiently large set of iid normalized and whitened input-output pairs, the better performing model under marginal likelihood is the one with larger input dimension.%
    \item \textbf{Double Descent}: For sufficiently small temperatures, GP cross-validation metrics exhibit double descent if and only if the mean squared error for the corresponding kernel regression task exhibits double descent (see \citet{liang2020just,liu2021kernel} for sufficient conditions).
\end{itemize}}

Along the way, we identify optimal choices of temperature (which can be interpreted as noise in the data) under a tempered posterior setup --- see Table \ref{tab:Summary} for a summary. 
In line with previous work on double descent curves \citep{belkin2019reconciling}, our objective is to investigate the behavior of the marginal likelihood with respect to model complexity, which is often given by the number of parameters in parametric settings \citep{d2020triple,derezinski2020exact,hastie2022surprises}). GPs are non-parametric, and while notions of \emph{effective dimension} do exist \citep{zhang2005learning,alaoui2015fast}, it is common to instead consider the input dimension in place of the number of parameters in this context  \citep{liang2020just,liu2021kernel}. We stress that the distinction between input dimension and model complexity should be taken into account when contrasting our results with existing work. 

Our results highlight that the common curse of dimensionality heuristic can be bypassed through an empirical Bayes procedure. Furthermore, the performance of optimally regularized GPs (under several metrics), can be improved with additional covariates (including synthetic ones). Our theory is supported by experiments performed on real large datasets. %
Our results also highlight that marginal likelihood and cross-validation metrics exhibit fundamentally different behavior for GPs, and requires separate analyses. 
Additional experiments, including the effect of ill-conditioned inputs, alternative data distributions, and choice of underlying kernel, are conducted in Appendix A. Details of the setup for each experiment are listed in Appendix B.

\begin{table*}
\centering
\begin{tabular}{|l|c|c|}
\hline
\rowcolor{Gray}\textbf{Performance Metric} & \textbf{Error Curve} & \textbf{Optimal $\gamma$}\\
\hline
Marginal Likelihood / Free Energy (\ref{eq:BayesEntropyGP}) & Monotone (Thm. \ref{thm:Main}) & eqn. (\ref{eq:OptGamma}) \\
Posterior Predictive $L^2$ Loss (\ref{eq:PPL2}) & Double Descent (Prop. \ref{prop:DD}) & $0$ \\
Posterior Predictive NLL (\ref{eq:PPNLL}) & Double Descent (Prop. \ref{prop:DD}) & $\mathbb{E}\|\bar{f}(x) - y\|^2$\\
\hline
\end{tabular}
\caption{\label{tab:Summary}Behavior of UQ performance metrics and optimal posterior temperature $\gamma$.\vspace{-.5cm}}
\end{table*}

\section{Background}
\label{sxn:background}

\subsection{Gaussian Processes}

A \emph{Gaussian process} is a stochastic process $f$ on $\mathbb{R}^{d}$
such that for any set of points $x_{1},\dots,x_{k}\in\mathbb{R}^{d}$,
$(f(x_{1}),\dots,f(x_{k}))$ is distributed as a multivariate Gaussian
random vector \citep[\S2.2]{williams2006gaussian}. Gaussian processes are completely determined by their
\emph{mean} and \emph{covariance functions}: if for any $x,x^\prime\in\mathbb{R}^{d}$,
$\mathbb{E}f(x)=m(x)$ and $\cov(f(x),f(x^\prime))=k(x,x^\prime)$, then we say
that $f\sim\GP(m,k)$. Inference for GPs is
informed by Bayes' rule: letting $(X_{i},Y_{i})_{i=1}^{n}$ denote
a collection of iid input-output
pairs, we impose the assumption that $Y_{i} =  f(X_{i})+\epsilon_{i}$ where each $\epsilon_{i}\sim\mathcal{N}(0,\gamma)$,
\new{yielding a Gaussian likelihood
$
p(Y\vert f,X)=(2\pi\gamma)^{-n/2}\exp(-\tfrac{1}{2\gamma}\| Y-f(X)\| ^{2}).
$
The parameter $\gamma$ is the \emph{temperature} of the model, and dictates the perceived accuracy of the labels. For example, taking $\gamma \to 0^+$ considers a model where the labels are noise-free. }

For the prior, we assume that $f\sim\GP(0,\lambda^{-1}k)$ for a
fixed covariance kernel $k$ and regularization parameter $\lambda>0$.
While other mean
functions $m$ can be considered, in the sequel we will
consider the case where $m\equiv0$. Indeed, if $m \neq 0$, then one can instead consider $\tilde{Y}_i = Y_i - m(X_i)$, so that $\tilde{Y}_i = \tilde{f}(X_i) + \epsilon_i$ and the corresponding prior for $\tilde{f}$ is zero-mean.
The Gram matrix $K_X \in \mathbb{R}^{n\times n}$ for $X$ has elements $K_X^{ij} = k(X_i,X_j)$. Let $\boldsymbol{x} = (x_1,\dots,x_m)$ denote a collection of $N$ points in $\mathbb{R}^d$,  $f(\boldsymbol{x}) = (f(x_1),\dots,f(x_m))$ and denote by $K_{\boldsymbol{x}} \in \mathbb{R}^{m\times m}$ and $k_{\boldsymbol{x}} \in \mathbb{R}^{n\times m}$ the matrices with elements $K_{\boldsymbol{x}}^{ij} = k(x_i,x_j)$ and $k_{\boldsymbol{x}}^{ij} = k(X_i,x_j)$. 

Given this setup, we are interested in several cross-validation metrics which quantify the uncertainty of the model.
The \textbf{posterior predictive distribution} (PPD) of $f(\boldsymbol{x})$ given $X,Y$ is \citep[pg. 16]{williams2006gaussian}
\[
f(\boldsymbol{x})\vert X, Y \sim \mathcal{N}(\bar{f}(\boldsymbol{x}),\lambda^{-1}\Sigma(\boldsymbol{x})),
\]
where $\bar{f}(\boldsymbol{x}) = k_{\boldsymbol{x}}^\top (K_X + \lambda \gamma I)^{-1} Y$ and $\Sigma(\boldsymbol{x}) = K_{\boldsymbol{x}} - k_{\boldsymbol{x}}^\top (K_X + \lambda \gamma I)^{-1} k_{\boldsymbol{x}}$.
This defines a posterior predictive distribution $\rho^\gamma$ on the GP $f$ given $X,Y$ (so $f~\vert~X,Y~\sim~\rho^\gamma$). 
If we let $\boldsymbol{y} = (y_1,\dots,y_m)$ denote a collection of $m$ associated \emph{test labels} corresponding to our test data $\boldsymbol{x}$, the \textbf{posterior predictive $L^2$ loss} (PPL2) 
is the quantity
\begin{equation}
\label{eq:PPL2}
\ell(\boldsymbol{x},\boldsymbol{y}) \coloneqq \mathbb{E}_{f\sim\rho^{\gamma}}\| f(\boldsymbol{x})-\boldsymbol{y}\| ^{2}=\| \bar{f}(\boldsymbol{x})-\boldsymbol{y}\|^{2}+\tfrac{1}{\lambda}\tr(\Sigma(\boldsymbol{x})).
\end{equation}
Closely related is the \textbf{posterior predictive negative log-likelihood} (PPNLL), given by
\ifdefined\techreport
\begin{equation}
\label{eq:PPNLL}
L(\boldsymbol{x},\boldsymbol{y}\vert X,Y)\coloneqq -\mathbb{E}_{f\sim\rho^{\gamma}}\log p(\boldsymbol{y}\vert f,\boldsymbol{x})=\tfrac{1}{2\gamma}\|\bar{f}(\boldsymbol{x})-\boldsymbol{y}\| ^{2}+\tfrac{1}{2\lambda\gamma}\tr(\Sigma(\boldsymbol{x}))+\tfrac{m}{2}  \log(2\pi\gamma).
\end{equation}
\else
\begin{multline}
\label{eq:PPNLL}
L(\boldsymbol{x},\boldsymbol{y}\vert X,Y)\coloneqq -\mathbb{E}_{f\sim\rho^{\gamma}}\log p(\boldsymbol{y}\vert f,\boldsymbol{x})\\=\tfrac{1}{2\gamma}\|\bar{f}(\boldsymbol{x})-\boldsymbol{y}\| ^{2}+\tfrac{1}{2\lambda\gamma}\tr(\Sigma(\boldsymbol{x}))+\tfrac{m}{2}  \log(2\pi\gamma).
\end{multline}
\fi

\subsection{Marginal Likelihood}

\label{sec:MarginalLikelihood}

The fundamental measure of model performance in Bayesian statistics is the \emph{marginal likelihood} (also known as the \emph{partition function} in statistical mechanics). 
Integrating the likelihood over the prior distribution $\pi$ provides a probability density of data under the prescribed model. 
Evaluating this density at the training data gives an indication of model suitability before posterior inference.
Hence, the marginal likelihood is $\mathcal{Z}_n = \mathbb{E}_{f \sim \pi} p(Y \vert f,X)$. 
A larger marginal likelihood is typically understood as an indication of superior model quality. 
The \textbf{Bayes free energy} $\mathcal{F}_n = -\log \mathcal{Z}_n$ is interpreted as an analogue of the test error, where smaller $\mathcal{F}_n$ is desired. 

The \textbf{marginal likelihood for a Gaussian process} is straightforward
to compute: since $Y_{i} = f(X_{i})+\epsilon_{i}$ under the likelihood, 
and $(f(X_{i}))_{i=1}^{n}\sim\mathcal{N}(0,\lambda^{-1}K_{X})$ under
the GP prior $\pi = \mathcal{GP}(0,\lambda^{-1}k)$, we have $Y_{i}\vert X\sim\mathcal{N}(0,\lambda^{-1}K_{X}+\gamma I)$,
and hence the Bayes free energy is given by \citep[eqn. (2.30)]{williams2006gaussian}
\ifdefined\techreport
\begin{equation}
\label{eq:BayesEntropyGP}
\mathcal{F}_{n}^{\gamma}=\tfrac{1}{2}\lambda Y^{\top}(K_{X}+\lambda\gamma I)^{-1}Y+\tfrac{1}{2}\log\det(K_{X}+\lambda\gamma I)-\tfrac{n}{2}\log\left(\tfrac{\lambda}{2\pi}\right).
\end{equation}
\else
\begin{multline}
\label{eq:BayesEntropyGP}
\mathcal{F}_{n}^{\gamma}=\tfrac{1}{2}\lambda Y^{\top}(K_{X}+\lambda\gamma I)^{-1}Y\\+\tfrac{1}{2}\log\det(K_{X}+\lambda\gamma I)-\tfrac{n}{2}\log\left(\tfrac{\lambda}{2\pi}\right).
\end{multline}
\fi
In practice, the hyperparameters $\lambda,\gamma$ are often tuned to minimize the Bayes free energy. This is an \emph{empirical Bayes procedure}, and typically achieves excellent results \citep{krivoruchko2019evaluation}.

The relationship between PPNLL and the marginal likelihood is perhaps best shown using cross-validation measures. 
Let $I$ be uniform on $\{1,\dots,k\}$ and let $\mathcal{T}$ be a random choice of $k$ indices from $\{1,\dots,n\}$
(the \emph{test set}). Let $\bar{\mathcal{T}}=\{1,\dots,n\}\backslash\mathcal{T}$
denote the corresponding \emph{training set}.
The leave-$k$-out cross-validation score under the PPNLL is defined by $S_{k}^\rho(X,Y)=\mathbb{E}L(X_{\mathcal{T}_{I}},Y_{\mathcal{T}_{I}}\vert X_{\bar{\mathcal{T}}},Y_{\bar{\mathcal{T}}})$. Letting $S_k(X,Y)$ denote the same quantity with the expectation in (\ref{eq:PPNLL}) over $\rho^\gamma$ replaced with an expectation over the prior, the Bayes free energy is the sum of all leave-$k$-out cross-validation scores \citep{fong2020marginal}, that is $\mathcal{F}_{n}^{\gamma}=\sum_{k=1}^{n}S_{k}(X,Y)$.
Therefore, the \textbf{mean Bayes free energy} (or mean free energy for brevity) $n^{-1}\mathcal{F}_{n}^{\gamma}$ can be interpreted as the average cross-validation score in the prior, instead of the posterior prediction. Similar connections can also be formulated in the PAC-Bayes framework \citep{germain2016pac}.

\subsection{Bayesian Linear Regression}

One of the most important special cases of GP regression is \emph{Bayesian linear regression}, obtained by taking $k_{\text{lin}}(x,x^\prime) = x^\top x^\prime$. 
As a special case of GPs, our results apply to Bayesian linear regression, directly extending double descent analysis into the Bayesian setting. 
By Mercer's Theorem \citep[\S4.3.1]{williams2006gaussian}, a realization of a GP $f$ has a series expansion in terms of the eigenfunctions of the kernel $k$. 
As the eigenfunctions of $k_{\text{lin}}$ are linear, $f\sim\mathcal{GP}(0,\lambda^{-1}k_{\text{lin}})$ if and only if 
\[
f(x)=w^{\top}x,\qquad w\sim\mathcal{N}(0,\lambda^{-1}).
\]
More generally, if $\phi:\mathbb{R}^{d}\to\mathbb{R}^{p}$ is a finite-dimensional feature map, then $f(x)=w^{\top}\phi(x)$, $w\sim\mathcal{N}(0,\lambda^{-1})$ is a GP with covariance kernel $k_{\phi}(x,y)=\phi(x)^{\top}\phi(y)$. 
This is the weight-space interpretation of Gaussian processes. 
In this setting, the posterior distribution over the weights satisfies $\rho^\gamma(w) = p(w \vert X,Y) \propto \exp(-\frac1{2\gamma}\|Y - \phi(X)w\|^2-\frac{\lambda}{2}\|w\|^2)$ and the marginal likelihood becomes
\ifdefined\techreport
\begin{equation}
\label{eq:MarginalLikelihoodBLR}
\mathcal{Z}_{n}^\gamma=\int_{\mathbb{R}^{p}}p(Y\vert X,w)\pi(w)\dd w
=\frac{\lambda^{d/2}}{\gamma^{n/2}(2\pi)^{\frac{1}{2}(n+d)}}\int_{\mathbb{R}^{p}}e^{-\frac{1}{2\gamma}\| Y-\phi(X)w\| ^{2}}e^{-\frac{\lambda}{2}\| w\| ^{2}}\dd w,
\end{equation}
\else
\begin{multline}
\label{eq:MarginalLikelihoodBLR}
\mathcal{Z}_{n}^\gamma=\int_{\mathbb{R}^{p}}p(Y\vert X,w)\pi(w)\dd w
\\=\frac{\lambda^{d/2}}{\gamma^{n/2}(2\pi)^{\frac{1}{2}(n+d)}}\int_{\mathbb{R}^{p}}e^{-\frac{1}{2\gamma}\| Y-\phi(X)w\| ^{2}}e^{-\frac{\lambda}{2}\| w\| ^{2}}\dd w,
\end{multline}
\fi
where $\phi(X)=(\phi(X_{i}))_{i=1}^{n}\in\mathbb{R}^{n\times p}$.
Under this interpretation, the role of $\lambda$ as a regularization parameter is clear. 
Note also that if $\lambda = \mu / \gamma$ for some $\mu > 0$, then the posterior $\rho^\gamma(w)$ depends on $\gamma$ as $(\rho^1(w))^{1/\gamma}$ (excluding normalizing constants). 
This is called a \emph{tempered posterior}; if $\gamma < 1$, the posterior is \emph{cold}, and it is \emph{warm} whenever $\gamma > 1$.

\section{Related Work}

\paragraph{Double Descent and Multiple Descent.} \emph{Double descent} is an observed phenomenon in the error curves of kernel regression, where the classical (U-shaped) bias-variance tradeoff in underparameterized regimes is accompanied by a curious monotone improvement in test error as the ratio $c$ of the number of datapoints to the ambient data dimension increases beyond $c=1$. The term was popularized in \citet{belkin2018understand,belkin2019reconciling}.
However, it had been observed in earlier reports \citep{opper1990ability,krogh1992generalization,dobriban2018high,loog2020brief}, and the existence of such non-monotonic behavior as a function of system control parameters should not be unexpected, given general considerations about different phases of learning that are well-known  from the statistical mechanics of learning \citep{EB01_BOOK,MM17_TR}. 
\new{An early precursor to double descent analysis came in the form of the \emph{Stein effect}, which establishes uniformly reduced risk when some degree of regularization is added \citep{strawderman2021charles}. Stein effects have been established for kernel regression in \citet{muandet2014kernel,chang2017data}.} Subsequent theoretical developments proved the existence of double descent error curves on various forms of linear regression \citep{bartlett2020benign,tsigler2020benign,hastie2022surprises,muthukumar2020harmless}, random features models \citep{gerace2020generalisation,liao2020random,holzmuller2020universality,mei2022generalization}, kernel regression \citep{liang2020just,liu2021kernel}, and classification tasks \citep{frei2022benign,wang2021benign,mignacco2020,deng2022model}, and other general feature maps \cite{loureiro2021learning}. 
For non-asymptotic results, subgaussian data is commonly assumed, yet other data distributions have also been considered \citep{derezinski2020exact}. 
Double descent error curves have also been observed in nearest neighbor models \citep{belkin2018overfitting}, decision trees \citep{belkin2019reconciling}, and state-of-the-art neural networks \citep{nakkiran2021deep,spigler2019jamming,geiger2020scaling}.
More recent developments have identified a large number of possible curves in kernel regression \citep{liu2021kernel}, including triple descent
\citep{adlam2020neural,d2020triple} and multiple descent for related volume-based metrics \citep{DKM20_CSSP_neurips}. 
Similar to our results, an optimal choice of regularization parameter can negate the double descent singularity and result in a monotone error curve  \citep{NIPS1991_8eefcfdf,liu2021kernel,nakkiran2020optimal,wu2020optimal}. 
While there does not appear to be clear consensus on a \emph{precise} definition of ``double descent,'' for our purposes, we say that an error curve $E(t)$ exhibits double descent if it contains a single global maximum away from zero at $t^\ast$, and decreases monotonically thereafter. 
This encompasses double descent as it appears in the works above, while excluding some misspecification settings and forms of multiple descent.

\paragraph{Learning Curves for Gaussian Processes.}
The study of error curves for GPs under posterior predictive losses has a long history (see \citet[\S7.3]{williams2006gaussian} and \citet{viering2021shape}). However, most results focus on rates of convergence of posterior predictive loss in the large data regime $n \to \infty$. The resulting error curve is called a \emph{learning curve}, as it tracks how fast the model learns with more data \citep{sollich1998learning,sollich2002learning,le2015asymptotic}. Of particular note are classical upper and lower bounds on posterior predictive loss \citep{opper1998general,sollich2002learning,williams2000upper}, which are similar in form to counterparts in the double descent literature \citep{holzmuller2020universality}. For example, some upper bounds have been obtained with respect to forms of \emph{effective dimension}, defined in terms of the Gram matrix \citep{zhang2005learning,alaoui2015fast}. Contraction rates in the posterior have also been examined \citep{lederer2019posterior}. In our work, we consider error curves over dimension rather than data, but we note that our techniques could also be used to study learning curves.

\paragraph{Cold Posteriors.}

Among the recently emergent phenomena encountered in Bayesian deep learning is the \emph{cold posterior effect} (CPE): the performance of Bayesian neural networks (BNNs) appears to improve for tempered posteriors when $\gamma \to 0^+$. 
This presents a challenge for uncertainty prediction: taking $\gamma \to 0^+$ concentrates the posterior around the \emph{maximum a posteriori} (MAP) point estimator, and so the CPE implies that optimal performance is achieved when there is little or no predicted uncertainty. Consequences in the setting of ensembling were discussed in \citet{wenzel2020good}, several authors have since sought to explain the phenomenon through data curation \citep{aitchison2020statistical}, data augmentation \citep{izmailov2021bayesian,fortuin2021bayesian}, and misspecified priors \citep{wenzel2020good}, although the CPE can still arise in isolation of each of these factors \citep{noci2021disentangling}. While our setup is too simple to examine the CPE at large, we find some common forms of posterior predictive loss are optimized as $\gamma \to 0^+$. %
\vspace{.25cm}

\section{Monotonicity in Bayes Free Energy}
\label{sec:Monotonicity}
\begin{figure*}[t]
\centering
\includegraphics[width=0.8\textwidth]{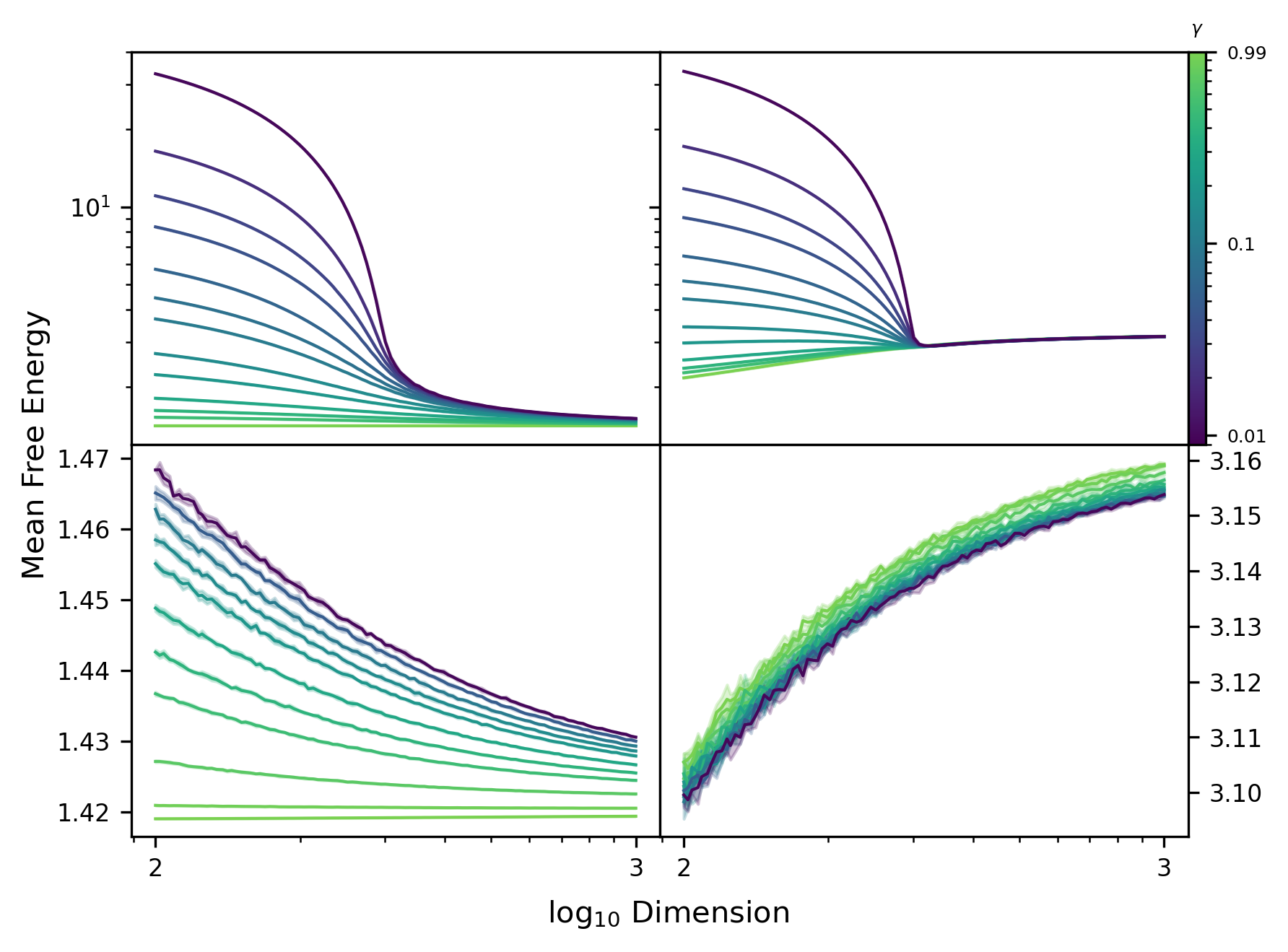}
\caption{\centering\label{fig:Monotone}Error curves for mean Bayes free energy $n^{-1}\mathcal{F}_n^\gamma$ for \textbf{synthetic data} under linear (top) and Gaussian (bottom) kernels, with $\lambda = \lambda^\ast$ (left; \textbf{monotone decreasing}) and $\lambda = 0.01$ (right; \textbf{increases at higher input dimensions}).\vspace{-.25cm}}%
\end{figure*}

In this section, we investigate the behavior of the Bayes free energy using the explicit expression in (\ref{eq:BayesEntropyGP}). %
First, to facilitate our analysis, we require the following assumption on the kernel $k$.
\vspace{.35cm}
\begin{assumption*}
The kernel $k$ is formed by a function $\kappa:\mathbb{R}\to\mathbb{R}$ that is continuously differentiable on $(0,\infty)$ and is one of the following two types:
\vspace{1cm}
\begin{enumerate}[label=(\Roman*)]
    \item \textbf{Inner product kernel:} $k(x,x^\prime) = \kappa(x^\top x^\prime / d)$ for $x,x^\prime\in \mathbb{R}^d$, where $\kappa$ is three-times continuously differentiable in a neighbourhood of zero, with $\kappa'(0) > 0$. Let
    \[
    \alpha = \kappa'(0),\qquad \beta = \kappa(1) - \kappa(0) - \kappa'(0). 
    \]
    \item \textbf{Radial basis kernel:} $k(x,x^\prime) = \kappa(\|x - x^\prime\|^2 / d)$ for $x,x^\prime\in\mathbb{R}^d$, where $\kappa$ is three-times continuously differentiable on $(0,\infty)$, with $\kappa'(2) < 0$. Let
    \[
    \alpha = -2\kappa'(2),\qquad \beta = \kappa(0) + 2\kappa'(2) - \kappa(2).
    \]
\end{enumerate}
\end{assumption*}
This assumption allows for many common covariance kernels used for GPs, including polynomial kernels $k(x,x^\prime)=(c+x^\top x^\prime /d)^p$, the exponential kernel $k(x,x^\prime)=\exp(x^\top x^\prime /d)$, the Gaussian kernel $k(x,x^\prime)=\exp(-\|x-x^\prime\|^2 / d)$, the multiquadric kernel $k(x,x^\prime) = (c+\|x-x^\prime\|^2 / d)^p$, the inverse multiquadric $k(x,x^\prime) = (c+\|x-x^\prime\|^2 / d)^{-p}$ kernels, and the Mat\'{e}rn kernels 
$$k(x,x^\prime)~=~(2^{\nu-1}\Gamma(\nu))^{-1} \|x~-~x^\prime\|^\nu K_{\nu}(\|x-x^\prime\|)$$
(where $K_\nu$ is the Bessel-$K$ function). Different bandwidths can also be incorporated through the choice of $\kappa$. Changing bandwidths between input dimensions can be incorporated into the variances of the data; to see the effect of this, we refer to Figure \ref{fig:IllConditioned} in Appendix A.
However, it does exclude some of the more recent and sophisticated kernel families, e.g., neural tangent kernels. Due to a result of \citet{karoui2010}, the Gram matrices of kernels satisfying this assumption exhibit limiting spectral behavior reminiscent of that for the linear kernel, $k(x,x^\prime) = c + x^\top x^\prime / d$. Roughly speaking, from the perspective of the marginal likelihood, we can treat GPs as Bayesian linear regression.

For our theory, we first consider the ``best-case scenario,'' where the prior is perfectly specified and its mean function $m$ is used to generate $Y$: $Y_i = m(X_i) + \epsilon_i$, where each $\epsilon_i$ is iid with zero mean and unit variance. By a change of variables, we can assume (without loss of generality) that $m \equiv 0$, so that $Y_i = \epsilon_i$, and is therefore independent of $X$. 
To apply the Marchenko-Pastur law from random matrix theory, we consider the large dataset -- large input dimension limit, where $n$ and $d$ scale linearly so that $d / n \to c \in (0,\infty)$. The inputs are assumed to have been \emph{whitened} and to be independent zero-mean random vectors with unit covariance. 
Under this limit, the sequence of mean Bayes entropies $n^{-1}\mathcal{F}_n^\gamma$, for each $n=1,2,\dots$, converges in expectation over the training set to a quantity $\mathcal{F}_\infty^\gamma$ which is more convenient to study. 
Our main result is presented in Theorem \ref{thm:Main}; the proof is delayed to 
Appendix G.

\begin{figure*}[t]
\centering

  \includegraphics[width=.45\textwidth]{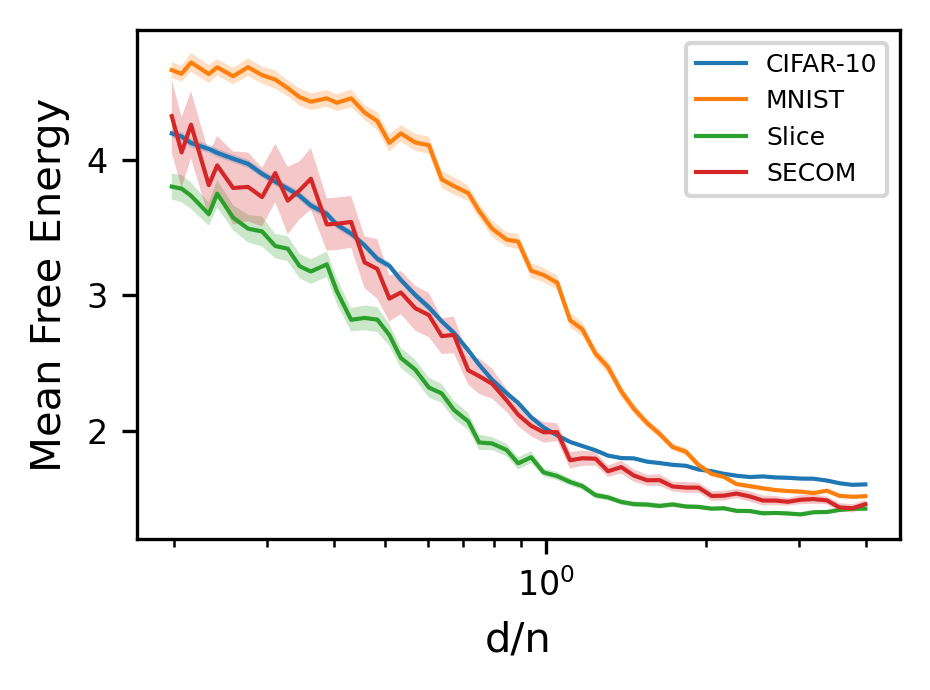}
  \includegraphics[width=.45\textwidth]{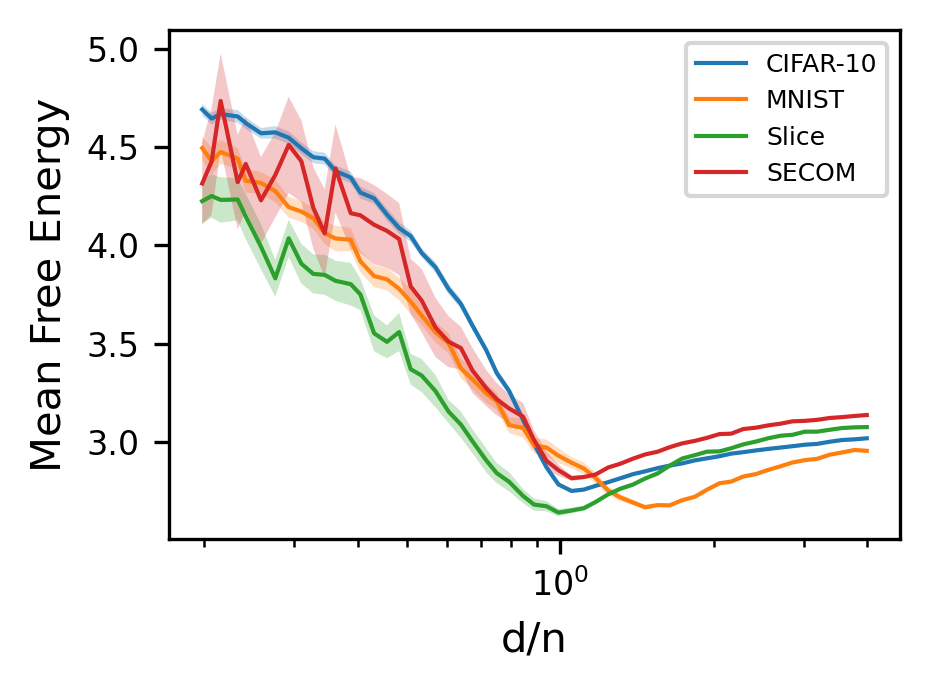}
\caption{\label{fig:RealBayesCIFAR}\label{fig:RealLinBayesCIFAR}\centering Error curves for mean Bayes free energy with $\gamma = 0.1$ under a range of datasets; linear kernel with $\lambda = \lambda^\ast$ (left), and $\lambda = 0.01$ (right); \textbf{curves for real data match Figure \ref{fig:Monotone} (top)}.}
\vspace{-.5cm}
\end{figure*}

\vspace{.5cm}
\begin{theorem}[Limiting Bayes Free Energy]
\label{thm:Main}
Let $X_1,X_2,\dots$ be independent and identically distributed zero-mean random vectors in $\mathbb{R}^d$ with unit covariance, satisfying $\mathbb{E}\|X_i\|^{5+\delta} < +\infty$ for some $\delta > 0$. For each $n=1,2,\dots$, let $\mathcal{F}_n^\gamma$ denote (\ref{eq:BayesEntropyGP}) applied to $X = (X_i)_{i=1}^n$ and $Y = (Y_i)_{i=1}^n$, with each $Y_i \sim \mathcal{N}(0,1)$. If $n,d\to\infty$ such that $d/n \to c \in (0,\infty)$, then
\[
\mathcal{F}_\infty^\gamma \coloneqq \lim_{n\to\infty} n^{-1} \mathbb{E}\mathcal{F}_n^\gamma,
\]
is well-defined. 
In this case,
\begin{enumerate}[label=(\alph*), ref=\ref{thm:Main}\alph*]
    \item \label{thm:MainGamma}If $\lambda = \mu / \gamma$ for some $\mu > 0$, there exists an optimal temperature $\gamma^\ast$ which minimizes $\mathcal{F}_\infty^\gamma$, which is given by
    \begin{equation}
    \label{eq:OptGamma}
    \gamma^\ast = c - 1 - \tfrac{c}{\alpha}(\beta + \mu) + \sqrt{(1 + \tfrac{c}{\alpha}(\beta+\mu+\alpha))^2 - 4c}.
    \end{equation}
\end{enumerate}
If the kernel $k$ depends on $\lambda$ such that $\alpha$ is constant in $\lambda$ and $\beta = \beta_0\lambda$ for $\beta_0 \in [0,1)$, then
\begin{enumerate}[label=(\alph*),resume, ref=\ref{thm:Main}\alph*]
    \item \label{thm:MainLambda}If $\gamma \in (0,1-\beta_0)$, there exists a unique optimal $\lambda^\ast > 0$ minimizing $\mathcal{F}_\infty^\gamma$ satisfying
    \begin{equation}
    \label{eq:OptLambda}
    \lambda^\ast = \frac{\alpha[(c+1)(\gamma+\beta_0) + \sqrt{(c-1)^2 + 4c(\gamma+\beta_0)^2}]}{c(1-(\gamma+\beta_0)^2)}.
    \end{equation}
    If $\gamma \geq 1-\beta_0$, then no such optimal $\lambda^\ast$ exists.
    \item \label{thm:MainMonotone}For any temperature $0 < \gamma < 1 - \beta_0$, at $\lambda = \lambda^\ast$, $\mathcal{F}_\infty^\gamma$ is \textbf{monotone decreasing} in $c \in (0,\infty)$.
\end{enumerate}
\end{theorem}

\begin{figure*}[t]
    \centering
    \includegraphics[width=\textwidth]{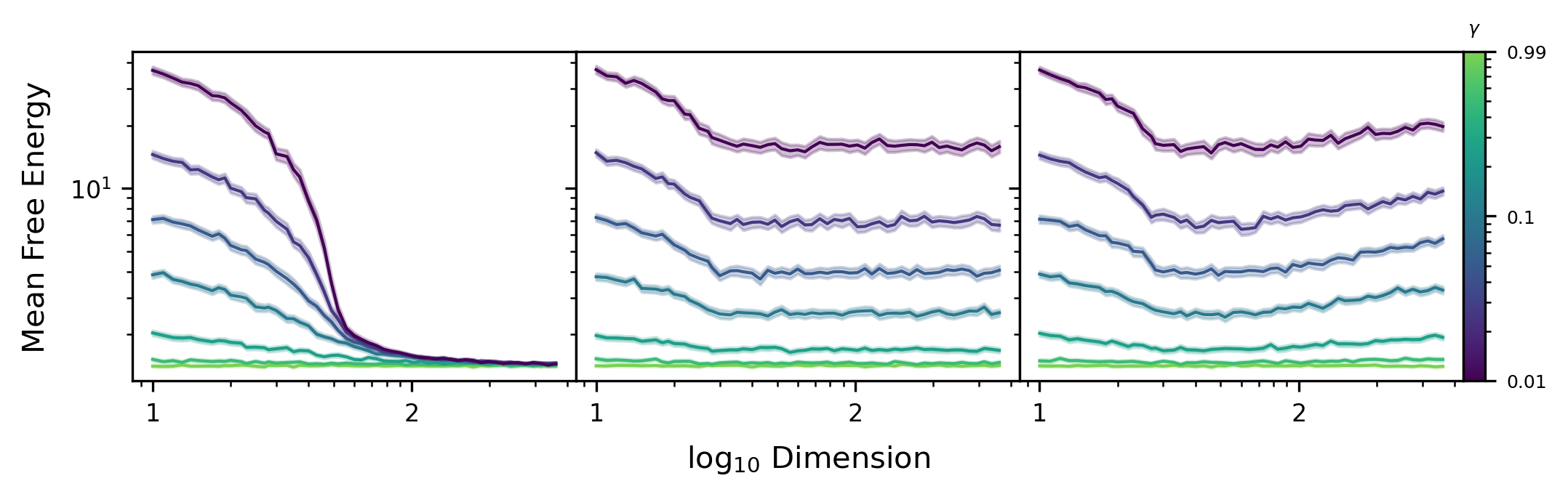}
    \caption{\centering Error curves for mean Bayes free energy under real data with Gaussian (left); repeated data (center); and zeroed data (right), under the linear kernel and $\lambda = \lambda^\ast$. \textbf{Only adding non-zero iid covariates improves model performance.}}
    \label{fig:RealDataAugmented}
\end{figure*}

The expression for the asymptotic Bayes free energy $\mathcal{F}_\infty^\gamma$ is provided in 
Appendix G.
To summarize, first, in the spirit of empirical Bayes, there exists an optimal $\lambda^\ast$ for the Gaussian prior which minimizes the asymptotic mean free energy. %
Under this setup, the choice of $\lambda$ which maximizes the marginal likelihood for a particular realization of $X,Y$ will converge almost surely to $\lambda^\ast$ as $n,d\to\infty$. %
Similar to \citet{nakkiran2020optimal,wu2020optimal}, we find that model performance under marginal likelihood improves monotonically with input dimension when $\lambda = \lambda^\ast$ %
{for a fixed amount of data.}
Indeed, for large $n, d$, $\mathbb{E}\mathcal{F}_n^\gamma \approx n \mathcal{F}_\infty^\gamma$ and $c \approx d / n$, so Theorem \ref{thm:MainMonotone} implies that the expected Bayes free energy decreases (approximately) monotonically with the input dimension, provided $n$ is fixed and the optimal regularizer $\lambda^\ast$ is chosen. %

\paragraph{Discussion of assumptions.} 
The assumption that the kernel scales with $\lambda$ is necessary using our techniques, as $\lambda^\ast$ cannot be computed explicitly otherwise. 
This trivially holds for the linear kernel ($\beta_0 = 0$), but most other choices of $\kappa$ can be made to satisfy the conditions of Theorem \ref{thm:Main} by taking $\kappa(x) \mapsto \eta^{-1} \kappa(\eta x)$, for appropriately chosen bandwidth $\eta \equiv \eta(\lambda)$.
For example, for the quadratic kernel, this gives $k(x,x^\prime) = (\lambda^{-1/2}+\lambda^{1/2}x^{\top}x^\prime)^{2}$. 
Effectively, this causes the regularization parameter to scale non-linearly in the prior kernel.
Even though this is required for our theory, we can empirically demonstrate this monotonicity also holds under the typical setup where $k$ does not change with $\lambda$.
In Figure~\ref{fig:Monotone}, we plot the mean free energy for synthetic Gaussian datasets of increasing dimension at both optimal and fixed values of $\lambda$ for the linear and Gaussian kernels. 
{Since $n$ is fixed,} in line with Theorem \ref{thm:MainMonotone}, the curves with optimally chosen $\lambda$ decrease monotonically with input dimension, while the curves for fixed $\lambda$ appear to increase when the dimension is large. 
Note, however, that the larger $\beta$ for the Gaussian kernel induces a significant regularizing effect. 
A light CPE appears for the Gaussian kernel when $\lambda$ is fixed, but does not seem to occur under $\lambda^\ast$.

While the assumption that $m = 0$ may appear too restrictive, in Appendix C, we show that $m$ is necessarily small when the data is normalized and whitened. Consequently, under a zero-mean prior, the marginal likelihood behaves similarly to our assumed scenario. This translates well in practice: under a similar setup to Figure~\ref{fig:Monotone}, the error curves corresponding to the linear kernel under a range of whitened benchmark datasets %
exhibit the predicted behavior (Figure~\ref{fig:RealLinBayesCIFAR}).

\paragraph{Synthetic covariates.} Since Theorem \ref{thm:Main} implies that performance under the marginal likelihood can improve as covariates are added, it is natural to ask whether an improvement will be seen if the data is augmented with synthetic covariates. To test this, we considered the first 30 covariates of the whitened \texttt{CT Slices} dataset obtained from the UCI Machine Learning Repository \citep{graf20112d}, and we augmented them with synthetic (iid standard normal) covariates; the first 30 covariates repeated; and zeros (for more details, see Appendix A).
While the first of these scenarios satisfies the conditions of Theorem \ref{thm:Main}, the second two do not, since the new data cannot be whitened such that its rows have unit covariance. %
Consequently, the behavior of the mean free energy reflects whether the assumptions of Theorem \ref{thm:Main} are satisfied: only the data with Gaussian covariates exhibits the same monotone decay. From a practical point of view, a surprising conclusion is reached: after optimal regularization, performance under marginal likelihood can be further improved by concatenating Gaussian noise to the~input.

\section{Double Descent in Posterior Predictive Loss}
\label{sec:DoubleDescent}

\begin{figure*}[t]
\centering
\includegraphics[width=0.8\textwidth]{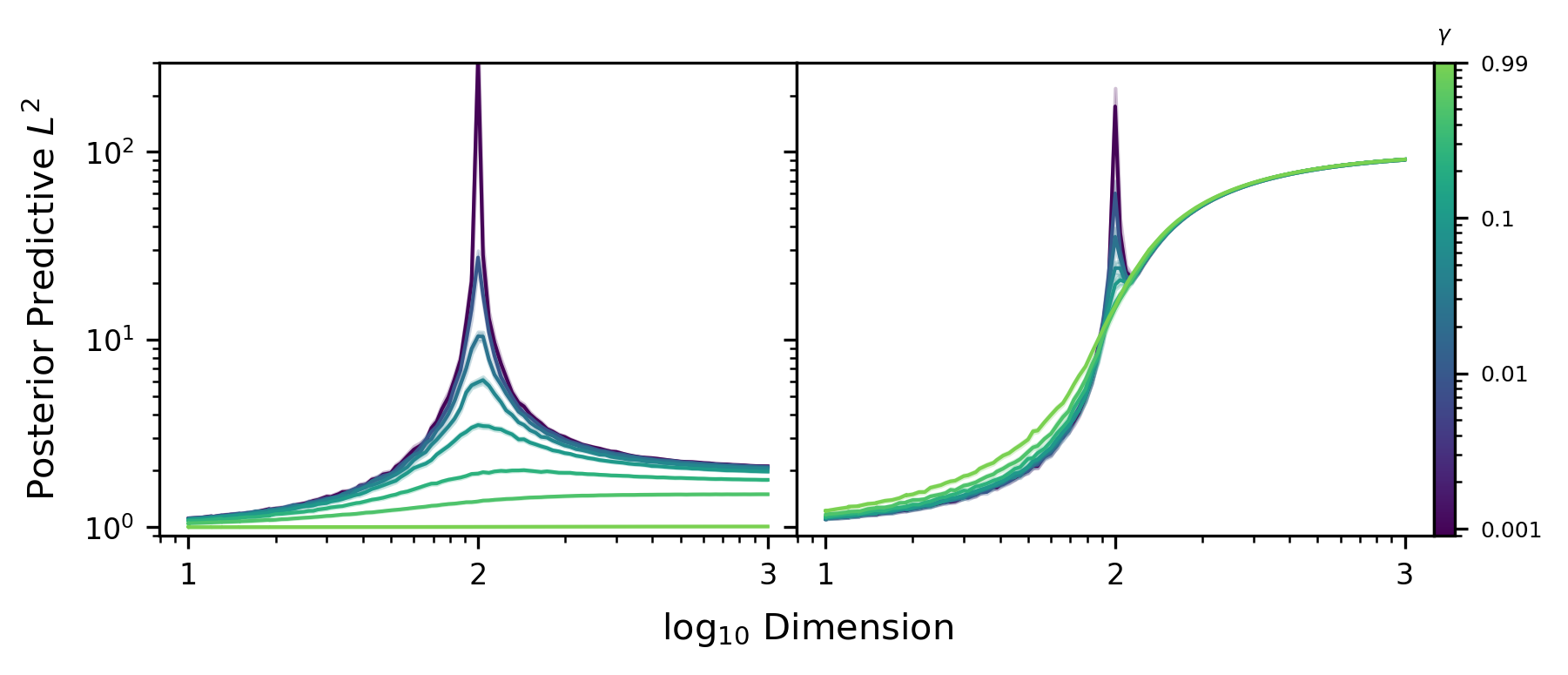}
\caption{\label{fig:PPL2}\centering Posterior predictive $L^2$ loss error curves for \textbf{synthetic data} exhibiting perturbed / tempered double descent under the linear kernel with $\lambda = \lambda^\ast$ (left), and $\lambda = 0.01$ (right).\vspace{-.25cm}}
\end{figure*}
\begin{figure*}[t]
\centering

  \centering
  \includegraphics[width=.45\textwidth]{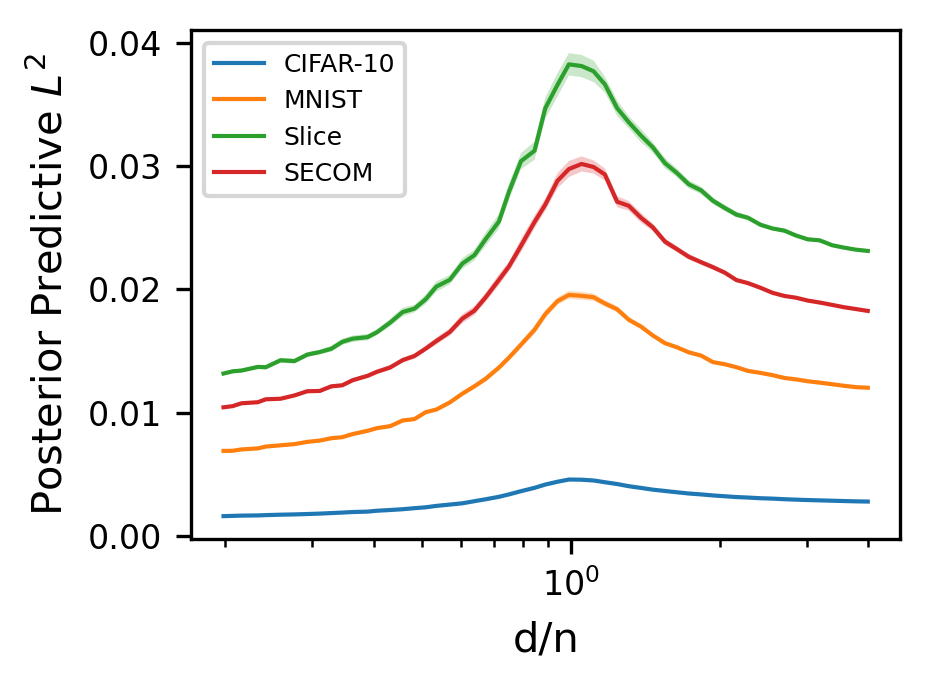}
  \includegraphics[width=.45\textwidth]{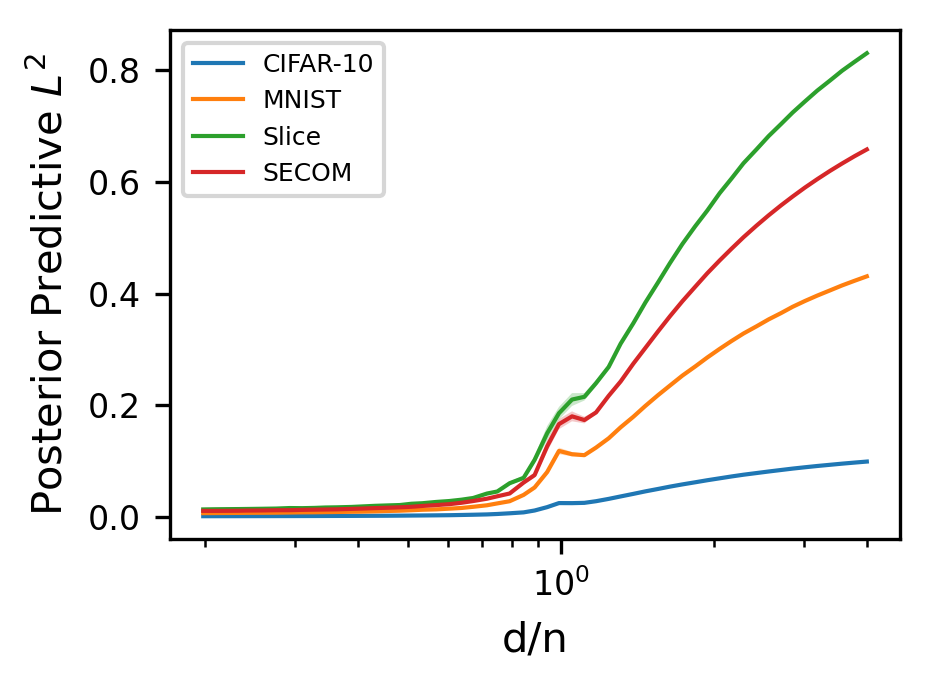}

\captionof{figure}{\label{fig:RealLinPPLCIFAR}\centering PPL2 loss with $\gamma = 0.1$ under the linear kernel with $\lambda = \lambda^\ast$ (left) and $\lambda = 0.01$ (right) on a range of datasets; \textbf{curves for real data match Figure \ref{fig:PPL2}.}}

\end{figure*}

In this section, we will demonstrate that, despite the connections between them, the marginal likelihood and posterior predictive loss can exhibit different qualitative behavior, with the posterior predictive losses potentially exhibiting a double descent phenomenon. 
Observe that the two forms of posterior predictive loss defined in (\ref{eq:PPL2}) and (\ref{eq:PPNLL}) can both be expressed in the form
\[
\mathcal{L} = c_1(\gamma) \underset{\text{MSE}}{\underbrace{\mathbb{E}\|\bar{f}(\boldsymbol{x}) - \boldsymbol{y}\|^2}} + c_2(\lambda,\gamma) \underset{\text{volume}}{\underbrace{\mathbb{E}\tr(\Sigma(\boldsymbol{x}))}} + c_3(\gamma).
\]
The first term is the mean-squared error (MSE) of the predictor $\bar{f}$, and is a well-studied object in the literature. In particular, \textbf{the MSE can exhibit double descent}, or other types of multiple descent error curves depending on $k$, in both ridgeless \citep{holzmuller2020universality,liang2020just} and general \citep{liu2021kernel} settings. 
On the other hand, the volume term has the uniform bound $\mathbb{E}\tr(\Sigma(\boldsymbol{x})) \leq m \mathbb{E} k(x,x)$, so provided $c_2$ is sufficiently small, the volume term should have little qualitative effect. 
The following is immediate.
\begin{proposition}
\label{prop:DD}
Assume that the MSE $\mathbb{E}\|\bar{f}(\boldsymbol{x}) - \boldsymbol{y}\|^2$ for Gaussian inputs $\boldsymbol{x}$ and labels $\boldsymbol{y}$ converges to an error curve $E(c)$ that exhibits double descent as $n \to\infty$ with $d \equiv d(n)$ satisfying $d(n) / n \to c \in (0,\infty)$. If there exists a function $\lambda(\gamma)$ such that $c_2(\lambda(\gamma),\gamma) / c_1(\gamma) \to 0$ as $\gamma \to 0^+$, then for any $\epsilon > 0$, there exists an error curve $\bar{E}(c)$ exhibiting double descent, a positive integer $N$, and $\gamma_0 > 0$ such that for any $0 < \gamma < \gamma_0$ and $n > N$, $|\mathcal{L}/c_1 - \bar{E}| < \epsilon$ at $d = d(n)$ and $\lambda = \lambda(\gamma)$. %
\end{proposition}

For \textbf{posterior predictive $L^2$ loss}, in the tempered posterior scenario where $\lambda = \mu / \gamma$, the MSE remains constant in $\gamma$, while $c_2/c_1 = \gamma / \mu$. Since the predictor $\bar{f}$ depends only on $\mu$, the optimal $\gamma$ in the tempered posterior scenario is realised as $\gamma \to 0^+$. 
In other words, under the posterior predictive $L^2$ loss, \emph{the best prediction of uncertainty is none at all}. This highlights a trivial form of CPE for PPL2 losses, suggesting it may not be suitable as a UQ metric. 
Here we shall empirically examine the linear kernel case; similar experiments for more general kernels are conducted in 
Appendix A.
In Figure~\ref{fig:PPL2}(right), we plot posterior predictive $L^2$ loss under the linear kernel on synthetic Gaussian data by varying $\mu$ while keeping $\gamma$ fixed.
We find that colder posteriors induce double descent on the error curves. %
Similar plots on a range of datasets are shown in Figure~\ref{fig:RealLinPPLCIFAR}(right), demonstrating that this behavior carries over to real data.
Choosing $\lambda = \lambda^\ast$ (the optimal $\lambda$ according to marginal likelihood) reveals a more typical set of regularized double descent curves; this is shown in Figure~\ref{fig:PPL2}(left) for synthetic data and Figure~\ref{fig:RealLinPPLCIFAR}(left) for a range of datasets. %
This is due to the monotone relationship between the volume term and $\lambda$, hence the error curve inherits its shape from the behavior of $\lambda^\ast$ (see Appendix A). This should be contrasted with the behavior of classification tasks observed by \citet{clarte2022study}, where the empirical Bayes estimator \emph{mitigates} double descent. 

In contrast, this phenomenon is not the case for \textbf{posterior predictive negative log-likelihood}.
Indeed, letting $\lambda = \mu / \gamma$ and optimizing the expectation of (\ref{eq:PPNLL}) in $\gamma$, the optimal $\gamma^\ast = m^{-1} \mathbb{E}\|\bar{f}(\boldsymbol{x}) - \boldsymbol{y}\|^2$. The expected optimal PPNLL is therefore
\ifdefined\techreport
\begin{equation}
\label{eq:PPNLLOpt}
-\mathbb{E}_{\boldsymbol{x},\boldsymbol{y}}\mathbb{E}_{f\sim\rho^{\gamma^\ast}}\log p(\boldsymbol{y}\vert f,\boldsymbol{x}) = \tfrac{1}{2} m [1 + \log(2\pi \mathbb{E}\|\bar{f}(\boldsymbol{x}) - \boldsymbol{y}\|^2)] + (2\mu)^{-1} \tr(\Sigma(\boldsymbol{x})).
\end{equation}
\else
\begin{multline}
\label{eq:PPNLLOpt}
-\mathbb{E}_{\boldsymbol{x},\boldsymbol{y}}\mathbb{E}_{f\sim\rho^{\gamma^\ast}}\log p(\boldsymbol{y}\vert f,\boldsymbol{x}) \\= \tfrac{1}{2} m [1 + \log(2\pi \mathbb{E}\|\bar{f}(\boldsymbol{x}) - \boldsymbol{y}\|^2)] + (2\mu)^{-1} \tr(\Sigma(\boldsymbol{x})).
\end{multline}
\fi
Otherwise, the PPNLL displays similar behavior to PPL2, as the two are related linearly.

\section{Conclusion}
Motivated by understanding the uncertainty properties of prediction from GP models, we have applied random matrix theory arguments and conducted several experiments to study the error curves of three UQ metrics for GPs. 
Contrary to classical heuristics, model performance under marginal likelihood/Bayes free energy improves monotonically with {input dimension} %
under appropriate regularization (Theorem~\ref{thm:Main}). 
However, Bayes free energy does not exhibit double descent. 
Instead, cross-validation loss inherits a double descent curve from non-UQ settings when the variance in the posterior distribution is sufficiently small (Proposition~\ref{prop:DD}). This was recently pointed out by \citet{lotfi2022bayesian}, where consequences and alternative metrics were proposed.
While our analysis was conducted under the assumption of a perfectly chosen prior mean, similar error curves appear to hold under small perturbations, which always holds for large whitened datasets.

\new{Although our contributions are predominantly theoretical, our results also have noteworthy practical consequences:
\begin{itemize}[leftmargin=*]
\item Tuning hyperparameters according to marginal likelihood is \textbf{essential} to ensuring good performance in higher dimensions, and it \textbf{completely negates the curse of dimensionality}.
\item When using $L^2$ losses as UQ metrics, care should be taken in view of the CPE. As such, \textbf{we do not recommend the use of this metric in lieu of other alternatives}. 
\item In agreement with the conjecture of \citet{wilson2020bayesian}, \textbf{increasing temperature mitigates the double descent singularity}.
\item Our experiments suggest that \textbf{further improvements beyond the optimization of hyperparameters may be possible with the addition of synthetic covariates}, although further investigation is needed before such a procedure can be universally recommended.
\end{itemize}}
RMT techniques are finding increasing adoption in machine learning settings \cite{couillet2011random,liao2021hessian,derezinski2021sparse}. 
In light of the surprisingly complex behavior on display, the fine-scale behavior our results demonstrate, and a surprising absence of UQ metrics in the double descent literature, we encourage increasing adoption of random matrix techniques for studying UQ / Bayesian metrics in double descent contexts and beyond. There are numerous avenues available for future work, including the incorporation of more general kernels (e.g., using results from \citet{fan2020spectra} to treat neural tangent kernels, which are commonly used as approximations for large-width neural networks), and different limiting regimes \citep{lu2022equivalence}.

\textbf{Acknowledgements.}
MM would like to acknowledge the IARPA (contract W911NF20C0035), NSF, and ONR for providing partial support of this work. This research was also partially
supported by the Australian Research Council through an Industrial Transformation Training Centre for Information Resilience (IC200100022) and the Australian Centre of Excellence for Mathematical and Statistical Frontiers (CE140100049).

\ifdefined\keepappendix

\clearpage

\ifdefined\techreport
\else
\onecolumn
\begin{center}

\Large \bf Monotonicity and Double Descent in Uncertainty Quantification with Gaussian Processes \vspace{3pt}\\ {\normalsize SUPPLEMENTARY DOCUMENT}

\end{center}
\fi

\appendix

\section{Additional Empirical Results}
\label{sec:Experiments}

In this section, we consider other factors not covered by our analysis in the main body of the paper. Full experimental details are given in Appendix G. 
\ifdefined\techreport
\else
\fi

\paragraph{\texttt{CT Slices} dataset.}
To demonstrate our procedure for working with real data, we first consider the \texttt{CT Slices} dataset obtained from the UCI Machine Learning Repository \citep{graf20112d}, comprised of $n = 53500$ images $X_1,\dots,X_n \in \mathbb{R}^d$ with $d = 385$ features, and corresponding scalar-valued labels $Y_1,\dots,Y_n \in \mathbb{R}$. This dataset is also used in Figure \ref{fig:RealDataAugmented}. The data was preprocessed in the following way: first, 17 features were observed to be linearly dependent on the others, and were removed to reveal $d = 368$ features. The sample mean $\mu_X$ and sample covariance matrix $\Sigma_X$ of $X_1,\dots,X_n$ were computed, and the input normalized by $X_i \mapsto \Sigma_X^{-1/2} (X_i - \mu_X)$. The labels were similarly normalized as $Y_i \mapsto (Y_i - \mu_Y)/\sigma_Y$, where $\mu_Y$ and $\sigma_Y$ are the sample mean and standard deviation of the labels, respectively. Under this preprocessing, $X$ and $Y$ are assumed to satisfy the conditions of Theorem \ref{thm:Main}. 

Figure \ref{fig:RealBayes} examines the mean Bayes free energy for the linear and Gaussian kernels, under the optimal choice of $\lambda^\ast$ from Theorem \ref{thm:Main}. This figure should be compared to the synthetic data examples shown in Figure \ref{fig:Monotone} (upper left and bottom left). Similarly, Figure \ref{fig:RealLinPPL} is the corresponding version of Figure \ref{fig:PPL2}. Notably, the characteristic behavior of all four plots is still prominent in the real data example.

\begin{figure}[b]
\centering
  \includegraphics[width=.45\textwidth]{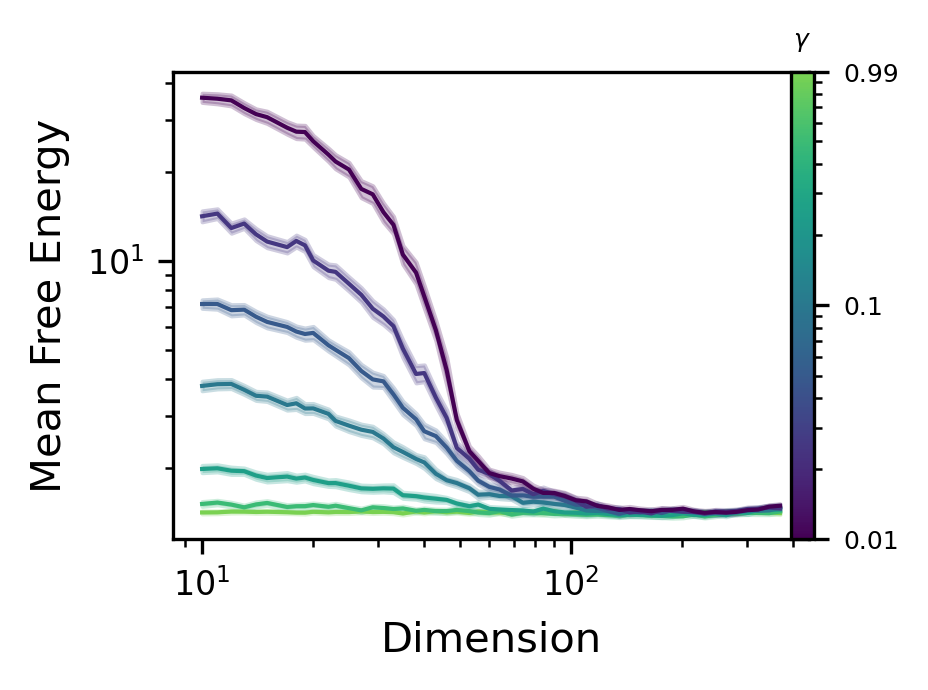}
  \includegraphics[width=.45\textwidth]{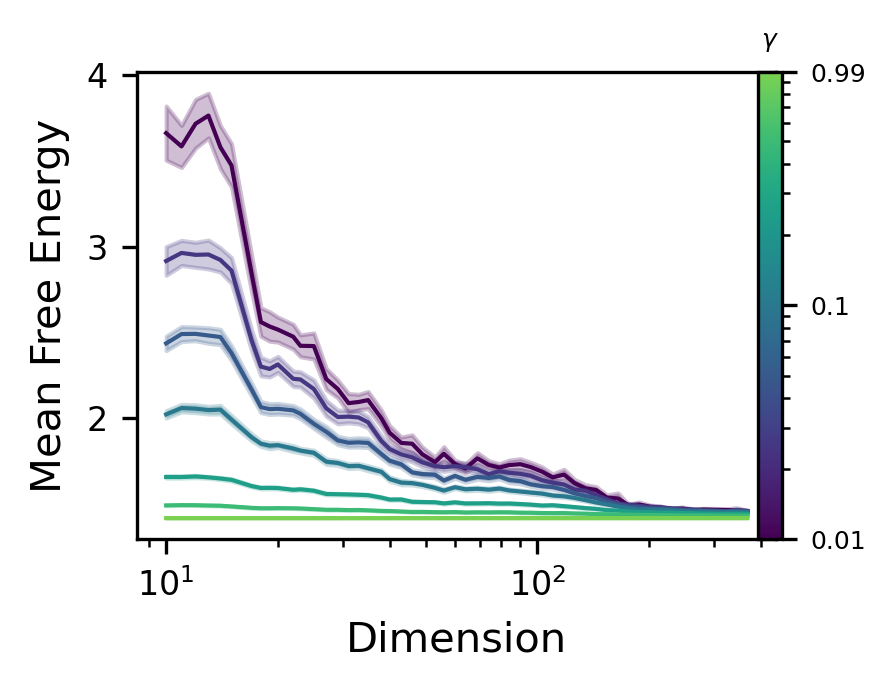}
  \caption{\label{fig:RealBayes}\centering Error curves for mean Bayes free energy under the CT Slices dataset; linear (left) and Gaussian (right) kernels; $\lambda = \lambda^\ast$}
\vspace{-.5cm}
\end{figure}

\begin{figure}[t]
\centering
  \includegraphics[width=.45\textwidth]{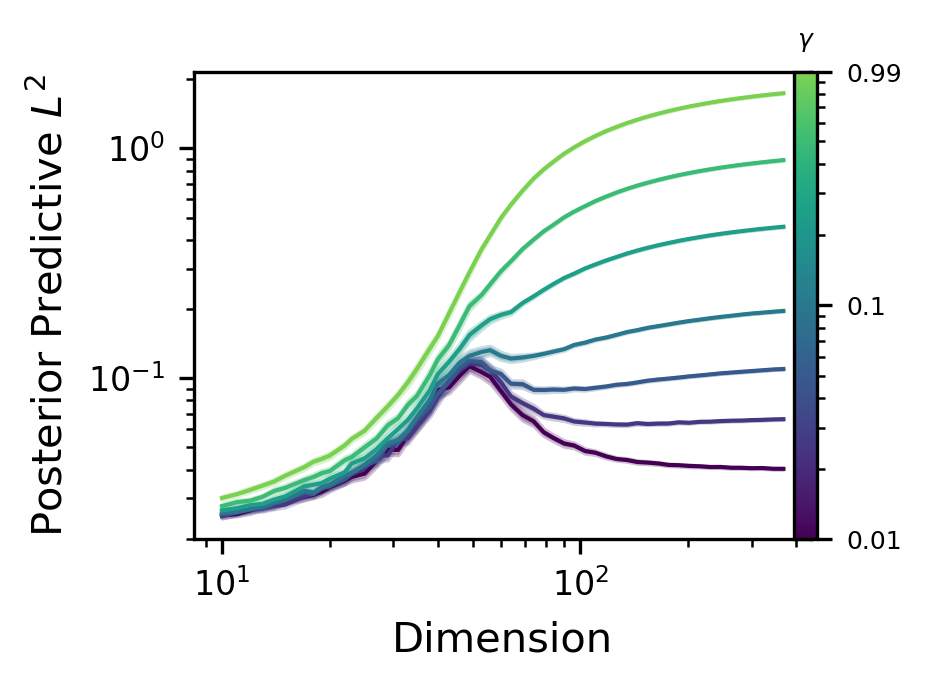}
  \includegraphics[width=.45\textwidth]{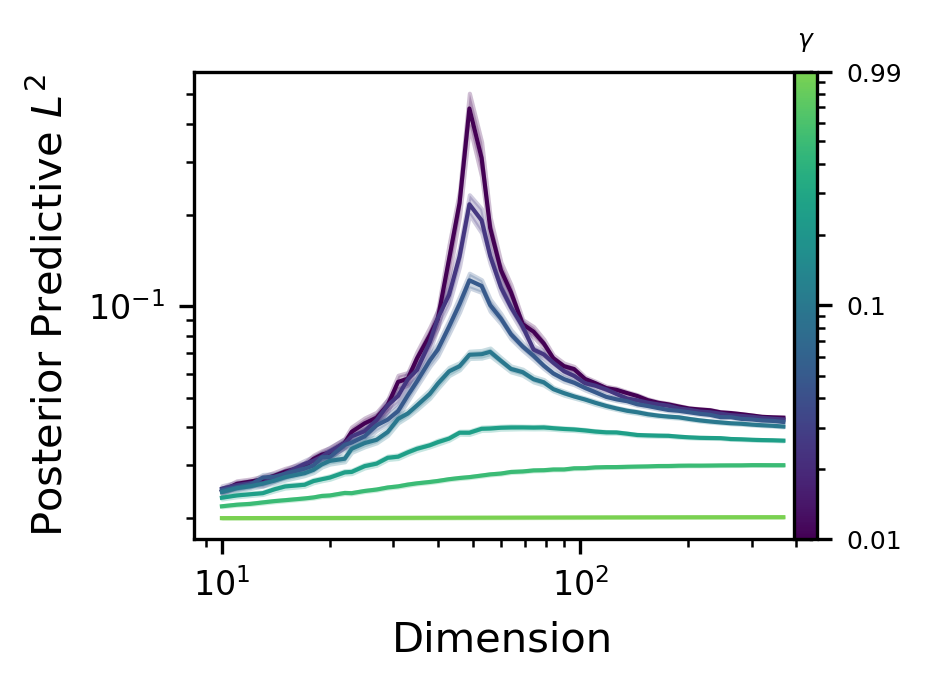}
  \caption{\label{fig:RealLinPPL}\centering PPL2 loss under the linear kernel with $\lambda = 0.01/\gamma$ (left) and $\lambda = \lambda^\ast$ (right) on the CT Slices dataset.}
\vspace{-.5cm}
\end{figure}

\paragraph{Image classification datasets}
We conducted parallel experiments on two larger benchmark datasets that are ubiquitous in the literature --- \texttt{MNIST} \cite{lecun1998gradient} and \texttt{CIFAR10} \cite{krizhevsky2009learning}. To this end, the \texttt{MNIST} and \texttt{CIFAR10} datasets were preprocessed in the same manner as the \texttt{CT Slices} dataset. Both datasets correspond to classification problems with $10$ class labels, however, for our purposes we consider the analogous regression problems over the class labels.

The \texttt{MNIST} training set is comprised of $60,000$ different $28\times28$ grayscale images of handwritten digits from 0-9. After preprocessing, $d=706$ of the 768 features were retained, and $n=175$ images were randomly sampled for use as the dataset. The mean free energy curves under the linear and Gaussian kernel under the optimal $\lambda = \lambda^\ast$, as well as the PPL2 curves for the optimal $\lambda$ and fixed $\mu$ are shown in Figures~\ref{fig:RealBayesMNIST} and \ref{fig:RealLinPPLMNIST}, respectively. Similarly, the \texttt{CIFAR10} training dataset contains $50,000$ different $32\times32$ color images, each with 3 channels. This corresponds to $3072$ features, of which $d = 3003$ were retained after preprocessing, and $n = 900$ images were randomly sampled as the for use as the dataset. Analogous images are presented in Figures~\ref{fig:RealBayesCIFAR} and \ref{fig:RealLinPPLCIFAR}.

It may seem surprising that the behavior of these models is so close to those of well-specified models, since there is no a priori reason to assume the mean of the data-generating process is zero. However, in Appendix B we demonstrate that this is merely a consequence of normalization of the response variables, and that such normalization forces tight control over the gradients of the mean function under expectation.
\begin{figure}[h]
\centering
  \centering
  \includegraphics[width=.45\textwidth]{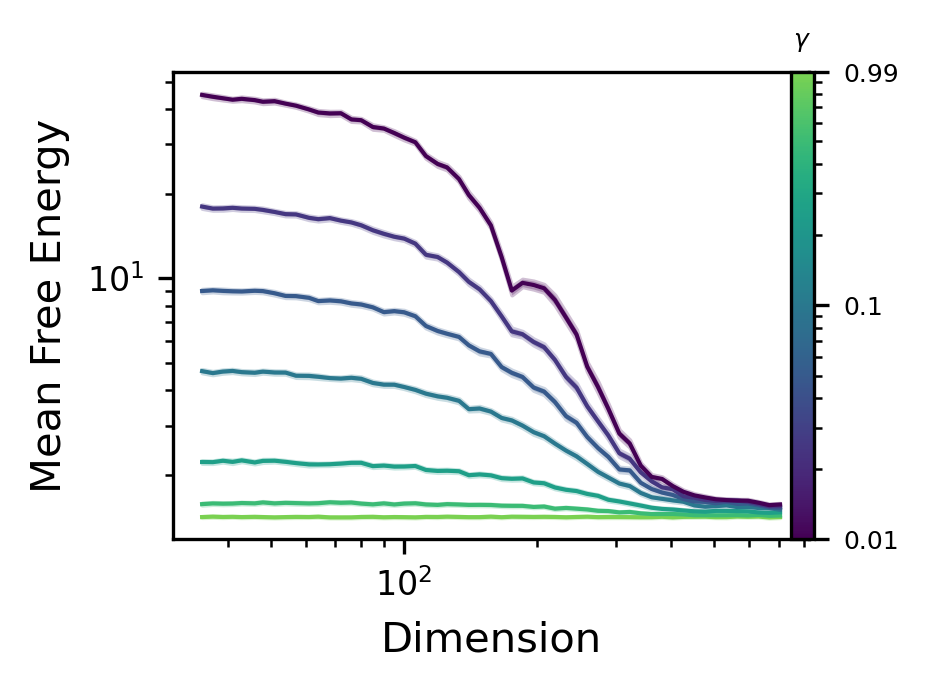}
  \includegraphics[width=.45\textwidth]{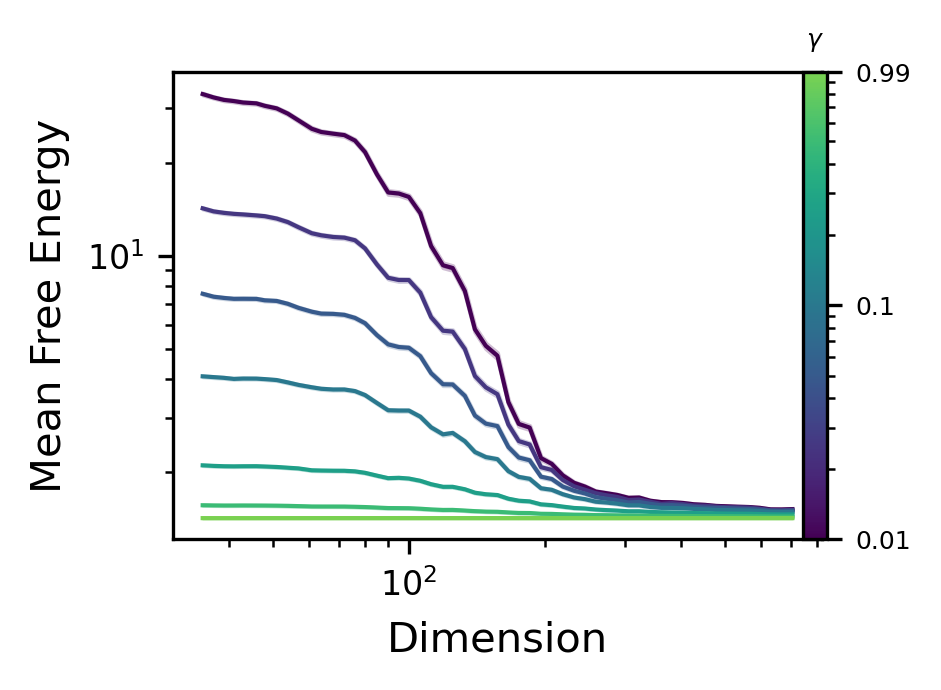}
  \captionof{figure}{\label{fig:RealBayesMNIST}\centering Error curves for mean Bayes free energy under the MNIST dataset; linear (left) and Gaussian (right) kernels; $\lambda = \lambda^\ast$}
\vspace{-.5cm}
\end{figure}

\begin{figure}[h]
\centering

  \includegraphics[width=.45\textwidth]{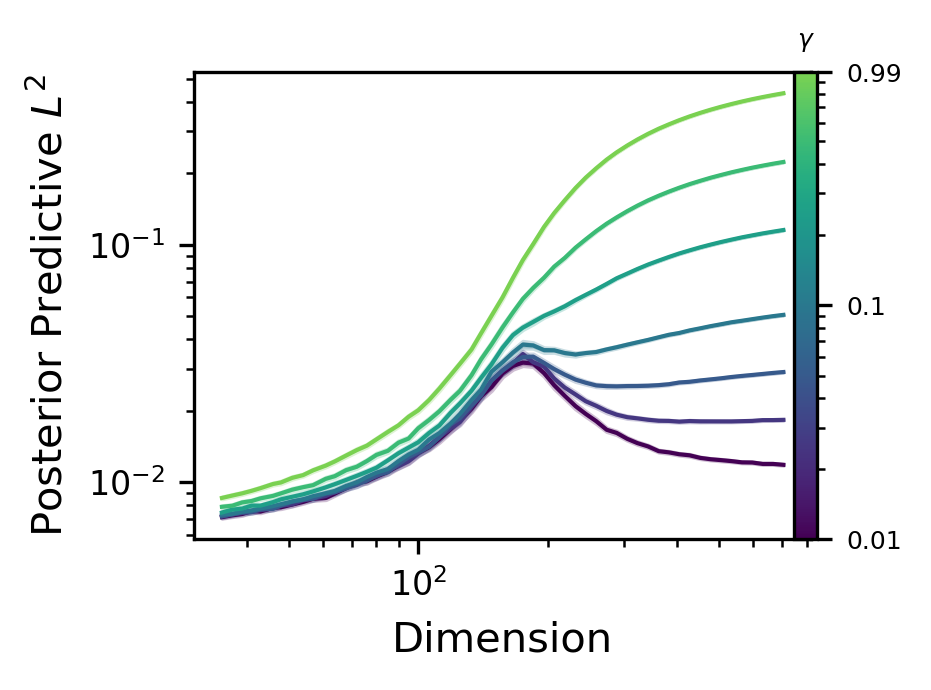}
  \includegraphics[width=.45\textwidth]{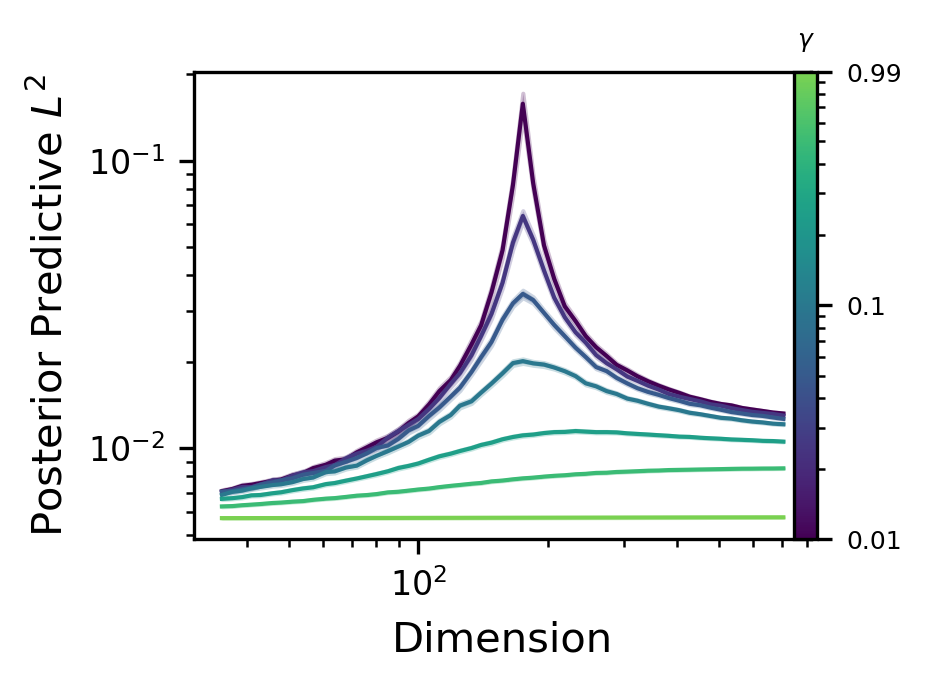}
  \captionof{figure}{\label{fig:RealLinPPLMNIST}\centering PPL2 loss under the linear kernel with $\lambda = 0.01/\gamma$ (left) and $\lambda = \lambda^\ast$ (right) on the MNIST dataset.}
\vspace{-.5cm}
\end{figure}

 \paragraph{Synthetic covariates.}
 From Theorem \ref{thm:Main}, one can conclude that performance under the marginal likelihood can increase as covariates are added. This begs the question: if the data is augmented with synthetic covariates, will this still result in a higher marginal likelihood? We have considered adding three different forms of synthetic covariates to the first 30 covariates of the whitened \texttt{CT Slices} dataset: 
 \begin{enumerate}[label=(\roman*)]
     \item \textbf{Gaussian white noise}: each $X_{ij}$ for $j > 30$ is drawn as an iid standard normal random variable;
     \item \textbf{Copied data}: the first 30 covariates are repeated, that is, for $j > 30$, each $X_{ij} = X_{i,(j-1) \text{ mod } 30 + 1}$, where mod denotes the modulus operator; and
     \item \textbf{Padded data}: each $X_{ij} = 0$ for $j > 30$.
 \end{enumerate}
 While case (i) satisfies the conditions of Theorem \ref{thm:Main}, cases (ii) and (iii) do not, as neither case can be whitened such that the rows of $X$ have unit covariance. In Figure \ref{fig:RealDataAugmented}, we repeat the experiment in the top left of Figure \ref{fig:Monotone} using these augmented datasets. The behavior of the mean Bayes free energy reflects whether the assumptions of Theorem \ref{thm:Main} are satisfied: while case (i) exhibits the same monotone decay, cases (ii) and (iii) do not.

\paragraph{Monotonicity in posterior predictive metrics.}

In these experiments, we consider posterior predictive metrics for synthetic data under optimal parameter choices. First, in Figure~\ref{fig:PPL2Opt}, the posterior predictive $L^2$ loss is optimized in $\lambda$, revealing a monotone decay in the dimension, analogous to \citet{nakkiran2020optimal,wu2020optimal}. In Figure~\ref{fig:PPNLL}, we plot error curves for the optimally tempered PPNLL metric (\ref{eq:PPNLLOpt}) under the linear kernel, revealing a monotonically increasing curve with input dimension when $\mu$ is fixed, and highlighting the need for appropriate regularization. If PPNLL is optimized in both $\gamma$ and $\mu$ simultaneously, the error curve becomes flat. 

\begin{figure}[b]
\centering
\begin{minipage}{.48\textwidth}
  \centering
  \includegraphics[width=.9\textwidth]{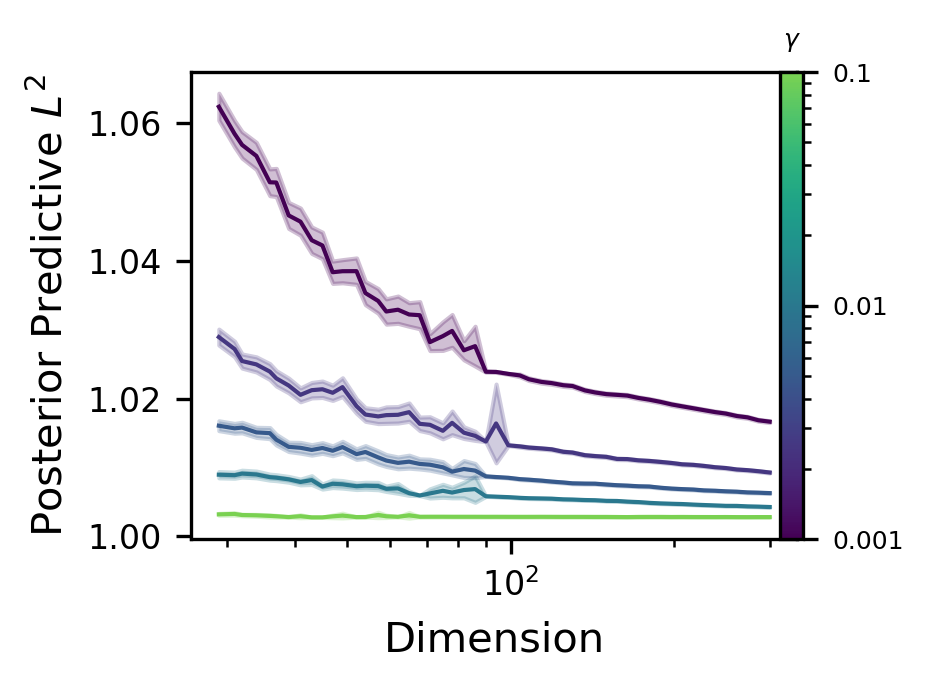}
  \captionof{figure}{\centering PPL2 optimized in $\lambda$; varying $\gamma$.}
  \label{fig:PPL2Opt}
\end{minipage}
\begin{minipage}{.48\textwidth}
  \centering
  \includegraphics[width=.9\textwidth]{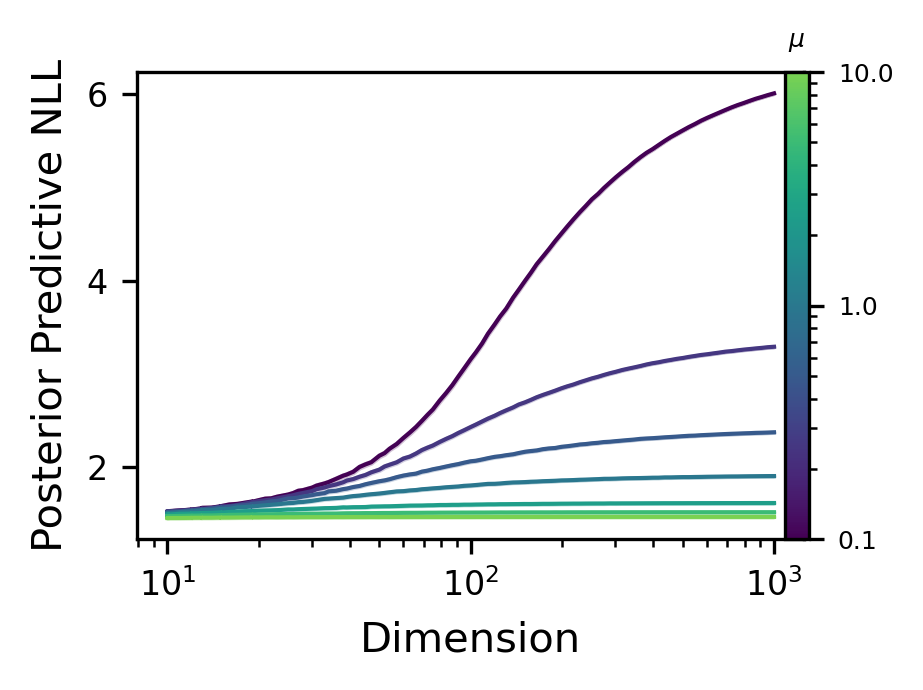}
  \captionof{figure}{\centering PPNLL optimized in $\gamma$ with $\lambda = \mu/\gamma$; varying $\mu$.}
  \label{fig:PPNLL}
\end{minipage}
\vspace{-.5cm}
\end{figure}

\paragraph{Prior misspecification.}
In our analysis, we have considered a practically optimal scenario where the prior is centered on the mean function of the labels (in other words, our prior concentrates on the correct solution). 
For more complex setups, where the prior is implicit and data-dependent, this may be possible, but is unlikely in general. For example, if the prior dictates \emph{a priori} knowledge, then a perfectly specified prior implies the underlying generative model for the labels is known in advance. 
Here, we assume that the mean function of the labels is nonzero, emulating a more realistic scenario. 
We restrict ourselves to the linear setting here, and we consider $Y_i = \theta_0 X_i + \epsilon_i$, ensuring that the correct mean function lies in the RKHS of the kernel. 
Figure \ref{fig:Misspec} illustrates the effect on Bayes free energy (with optimal $\lambda^\ast$). 
From left to right, small $\theta_0 = d^{-1/2} \boldsymbol{1}$, large $\theta_0 = n d^{-1/2} \boldsymbol{1}$, and growing $\theta_0 = \boldsymbol{1}$ perturbations are considered. 
For small values of the perturbation, the monotonicity of the error curve is not affected in a meaningful way. While the zero-mean assumption may seem restrictive, we demonstrate that this scenario will always hold asymptotically, provided the data is normalized and whitened (see Appendix B).
For larger perturbations, however, we see a horizontal ``double-ascent'' (or reverse double descent) error curve. 
A growing perturbation also results in a double-ascent curve, but with increasing Bayes free energy once the input dimension is sufficiently~large. 

\begin{figure}[t]
\includegraphics[width=\textwidth]{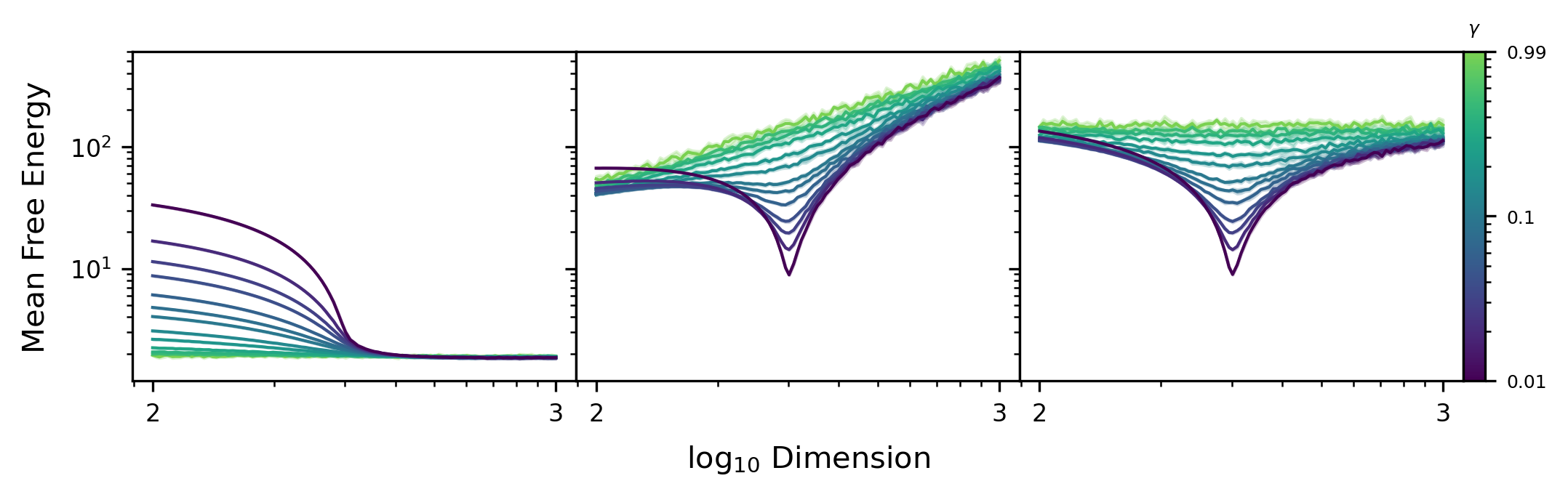}
\caption{\label{fig:Misspec}Optimal mean Bayes free energy with low (left), increasing (center) and high (right) levels of prior misspecification under the linear kernel.}
\end{figure}

\paragraph{Regularity of the kernel.}
The regularity of the kernel plays a key role in the regularity, and consequently, the quality of the predictor.
In particular, less regular predictors tend to revert to the prior more quickly away from the training data. 
The Mat\'{e}rn kernel family is noteworthy for its capacity to adjust the regularity of predictors through the parameter $\nu$, whereby realizations of a Gaussian process with Mat\'{e}rn covariance are at most $[\nu]$-times differentiable (see page 85 of \cite{williams2006gaussian}). 
In Figure \ref{fig:Matern}, we plot the Bayes free energy for fixed $n$ = 300 and $\gamma = 0.01$ with optimal $\lambda^\ast$ over input dimensions $d \in [100,1000]$ and $\nu \in [0.5,100]$. 
As $\nu$ decreases, the curves become flatter, suggesting the effect of dimension is reduced. 

\begin{figure}[t]
\centering
\begin{minipage}{.48\textwidth}
  \centering
  \includegraphics[width=.9\textwidth]{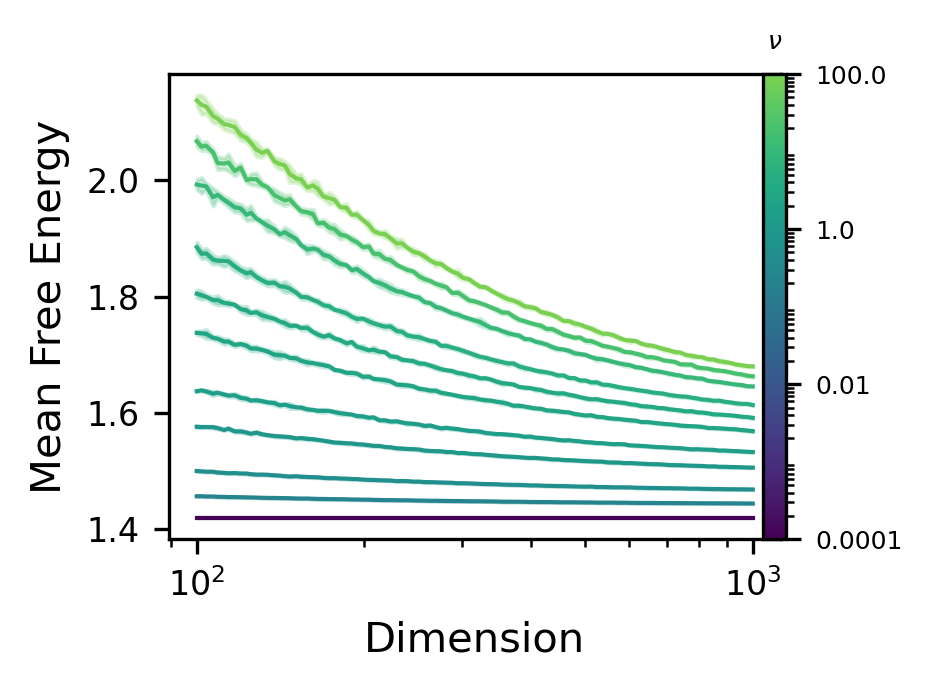}
  \captionof{figure}{\centering Effect of varying the regularity parameter $\nu$ in the Mat\'ern kernel}
  \label{fig:Matern}
\end{minipage}
\begin{minipage}{.48\textwidth}
  \centering
  \includegraphics[width=.9\textwidth]{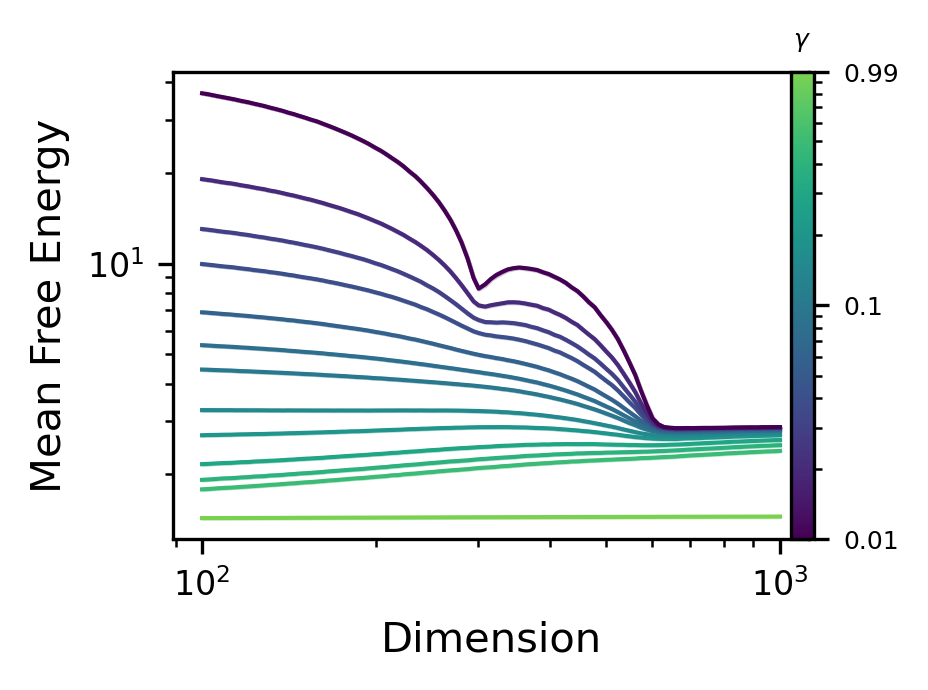}
  \captionof{figure}{\centering Effect of irregular spectra in $X$ under the linear kernel}
  \label{fig:IllConditioned}
\end{minipage}
\end{figure}

\paragraph{Ill-conditioned data.}
Our theoretical analysis considers only the case where the data has been whitened, that is, where each row of $X$ has unit covariance. 
It is known that more interesting behavior can occur depending on the spectrum of eigenvalues of the covariance matrix, including multiple descent \citep{nakkiran2020optimal,hastie2022surprises}, and this appears to be robust to other volume-based objectives \citep{DKM20_CSSP_neurips}. Recent work has also tied model performance to particular classes of spectral distributions, including power laws \cite{liao2021hessian,martin2020heavy,mahoney2019traditional}.
In Figure \ref{fig:IllConditioned}, we consider an isotropic ill-conditioned covariance matrix $\cov(X_i) = \Sigma$ where $\Sigma = \mbox{diag}((10)_{i=1}^{d/2},(1/10)_{i=1}^{d/2})$. 
Under the linear kernel, for fixed $\lambda$, the error curve is similar to the isotropic setting. 
However, at $\lambda=\lambda^\ast$, we find that the mean Bayes free energy can exhibit non-monotonic behavior at low temperatures. %

\paragraph{Scaling dimension nonlinearly with data.}
An interesting consequence of the monotonic error curve in the Bayes free energy is that the inclusion of additional data may be harmful if the input dimension is increased at a slower rate $d = \mathcal{O}(n^{\xi})$ for $\xi < 1$ (or beneficial if $\xi > 1$). This effect is illustrated in Figure \ref{fig:Spider}, where the normalized Bayes free energy $n^{-1}\mathcal{F}_n^\gamma$ is plotted for the linear and Gaussian kernels at the optimal $\lambda^\ast$ over $n \in [300,3000]$ with $d = 2^{10(1-\xi)} n^\xi$.
\begin{figure}[t]
\centering
\includegraphics[width=0.8\textwidth]{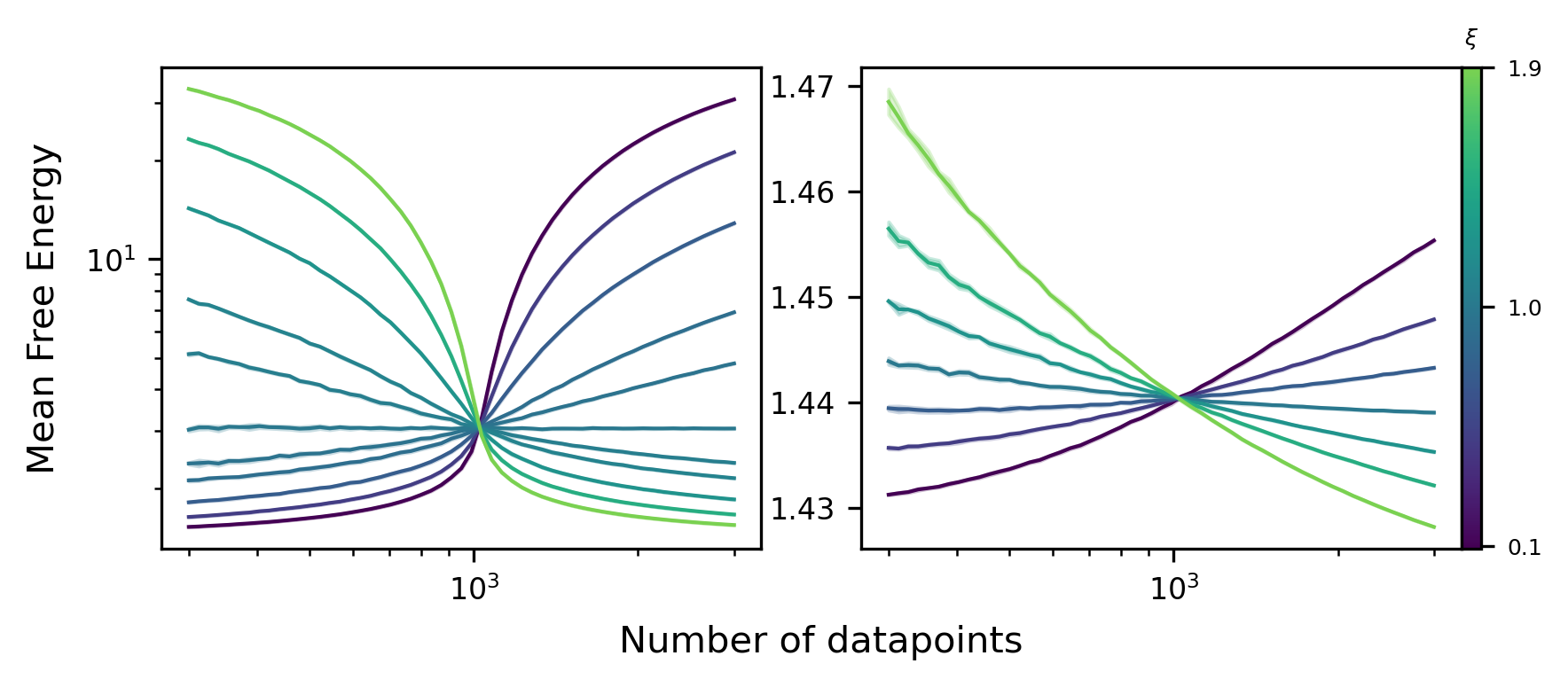}
\caption{\centering\label{fig:Spider}Error curves for mean Bayes free energy $n^{-1} \mathcal{F}_n^\gamma$ under linear (left) and Gaussian (right) kernels and $\lambda = \lambda^\ast$, for dimension scaling with data as $\mathcal{O}(n^\xi)$\vspace{-.25cm}}
\end{figure}

\paragraph{Effect of noise distribution}

Each experiment has also assumed that the labels are standard normal. If this is not the case, but the labels are still assumed to be iid, have zero mean and are uncorrelated with the inputs (correctly specified prior), then the expected mean Bayes free energy satisfies
\begin{align*}
n^{-1}\mathbb{E}\mathcal{F}_n^\gamma &= \frac{\lambda}{2n} \sum_{i,j=1}^n \mathbb{E}[Y_i Y_j Q_{ij}] + \frac1{2n} \mathbb{E}\log \det (K_X + \lambda\gamma I) - \frac12 \log\left(\frac{\lambda}{2\pi}\right),\\
&= \frac{\lambda}{2n} \sigma^2  \mathbb{E}\tr(Q) + \frac1{2n} \mathbb{E}\log \det (K_X + \lambda\gamma I) - \frac12 \log\left(\frac{\lambda}{2\pi}\right),
\end{align*}
where $Q = (K_X + \lambda\gamma I)^{-1}$ and $\sigma^2 = \mathbb{E}[Y_i^2]$. Therefore, only the variance in the labels contributes to $n^{-1}\mathbb{E}\mathcal{F}_n^\gamma$ (other features of the distribution of the noise contribute to the higher order moments of $\mathcal{F}_n^\gamma$). In Figure \ref{fig:VarLin}, we examine the effect that different variances in the label noise have on the mean Bayes free energy. %
Normally distributed $Y_i$ were considered, with variances ranging from $0.1$ to $10$.

\begin{figure}[h]
\centering
  \includegraphics[width=.45\textwidth]{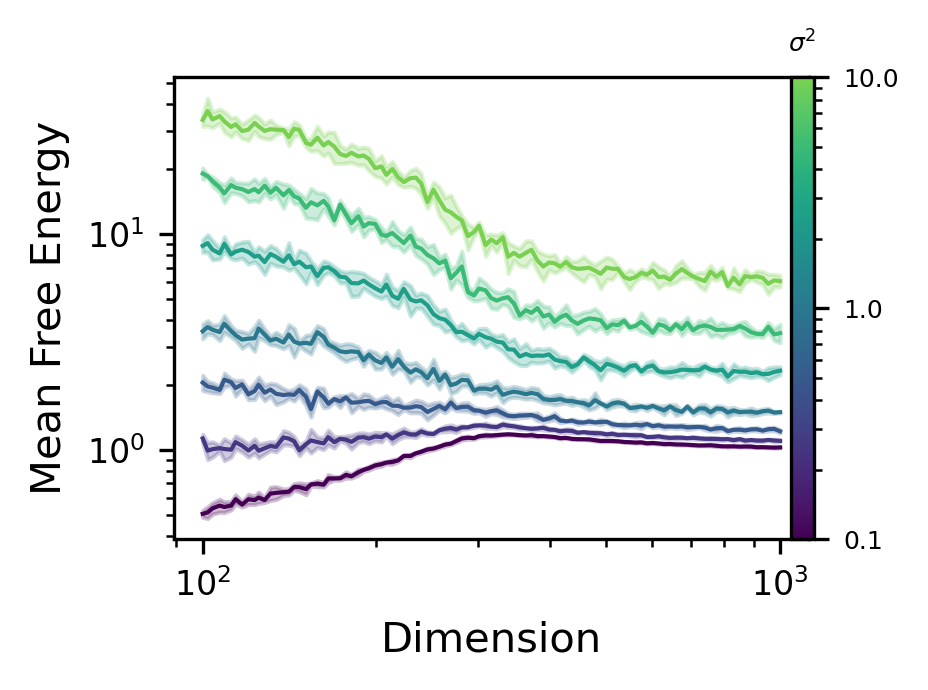}
  \includegraphics[width=.45\textwidth]{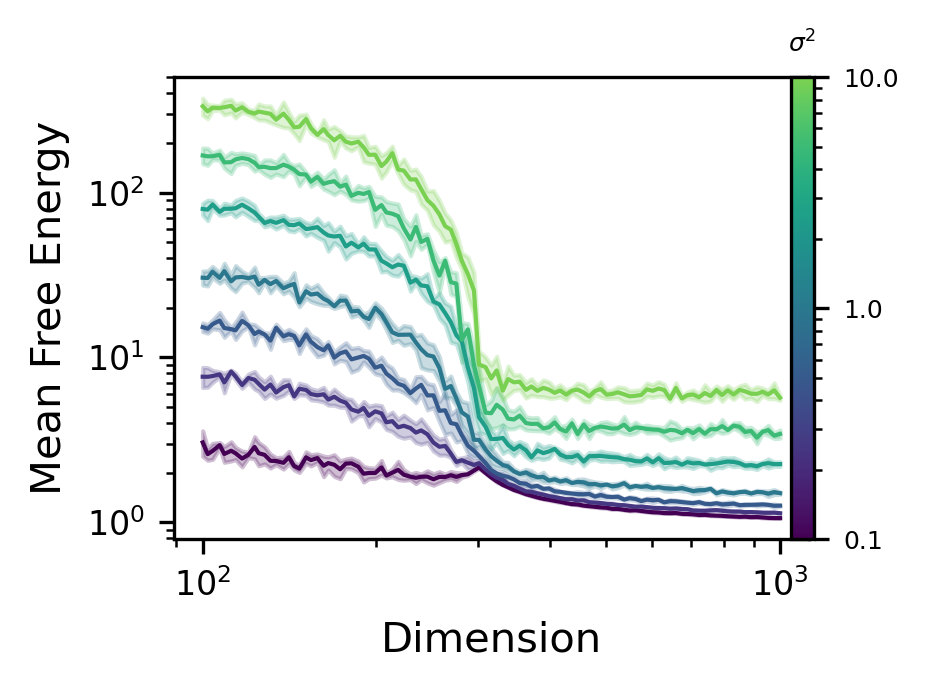}
  \caption{\label{fig:VarLin}\centering Error curves for mean Bayes free energy under linear kernel with $\lambda = \lambda^\ast$, $\gamma = 0.1$ (left) / $\gamma = 0.01$ (right) and different variances in the label data.}
\vspace{-.5cm}
\end{figure}

\paragraph{Posterior predictive loss with Gaussian kernel}

Figures \ref{fig:PPL2GaussMu} and \ref{fig:PPL2GaussOpt} examine the effect of varying $\gamma$ on the posterior predictive $L^2$ loss varying over $d$, under the Gaussian kernel. These figures should be contrasted with the linear kernel case presented as Figure \ref{fig:PPL2} (left and right, respectively). Note that the significant regularizing effect when $\beta > 0$ prohibits the double descent behavior found in the linear kernel case.

\begin{figure}[h]
\centering
\begin{minipage}{.48\textwidth}
  \centering
  \includegraphics[width=.9\textwidth]{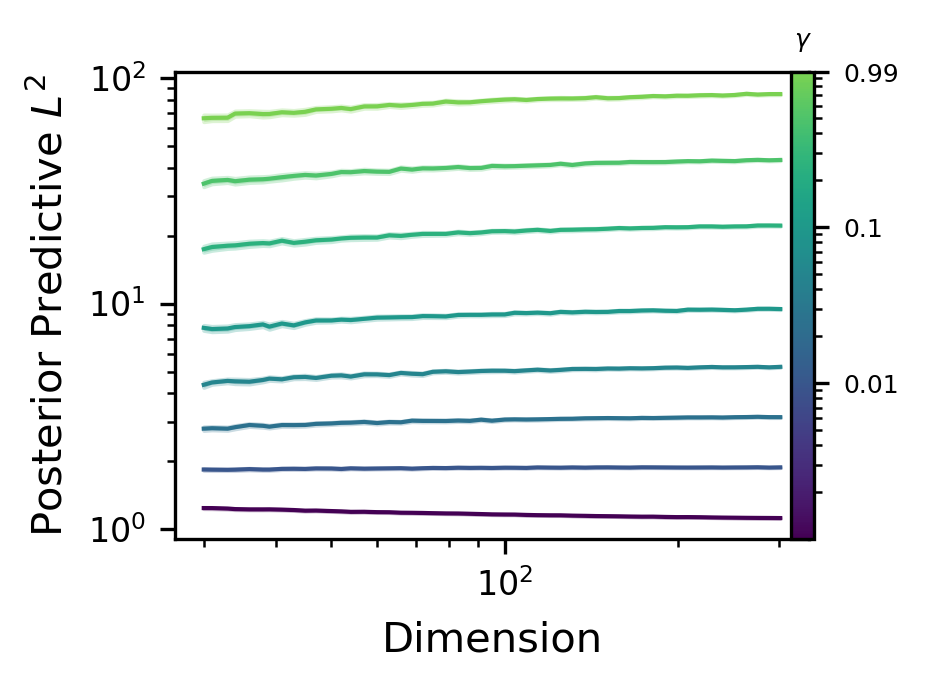}
  \captionof{figure}{\label{fig:PPL2GaussMu}\centering Posterior predictive $L^2$ loss under the Gaussian kernel with $\lambda = 0.01/\gamma$}
\end{minipage}
\begin{minipage}{.48\textwidth}
  \centering
  \includegraphics[width=.9\textwidth]{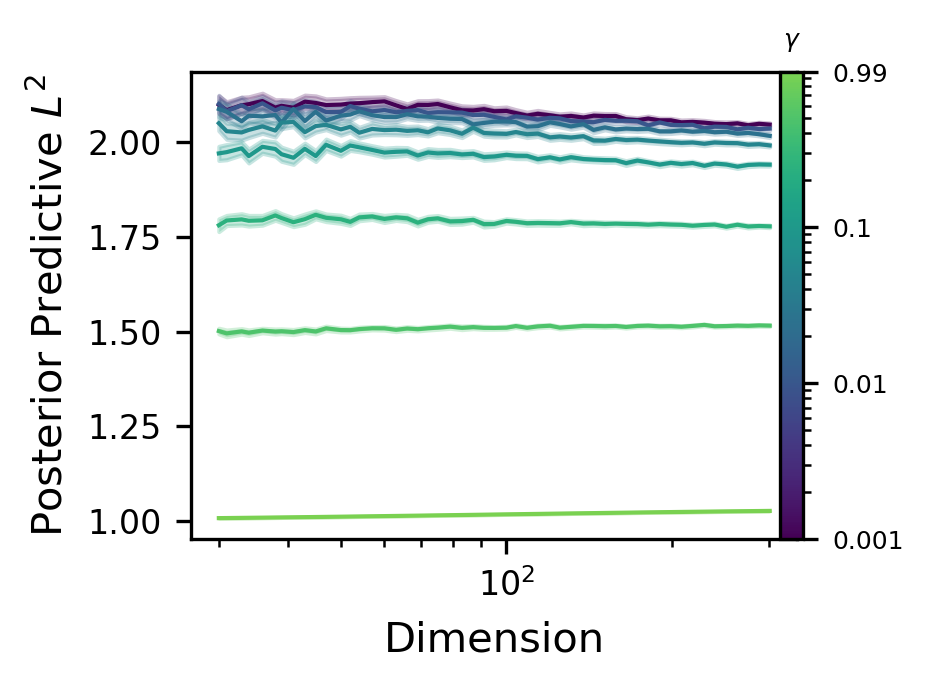}
  \captionof{figure}{\label{fig:PPL2GaussOpt}\centering Posterior predictive $L^2$ loss under the Gaussian kernel with $\lambda = \lambda^\ast$}
\end{minipage}
\vspace{-.5cm}
\end{figure}

\paragraph{Visualizing $\lambda^\ast$}
Figures \ref{fig:LambdaOptLin} and \ref{fig:LambdaOptGauss} plot the values of $\lambda^\ast$ versus $c$ over different values of $\gamma$, for the linear and Gaussian kernels, respectively. Once again, the sharp trough formed at $d=n$ in Figure \ref{fig:LambdaOptGauss} is significantly dampened by the regularizing effect of $\beta > 0$.
\begin{figure}[h]
\centering
\begin{minipage}{.48\textwidth}
  \centering
  \includegraphics[width=.9\textwidth]{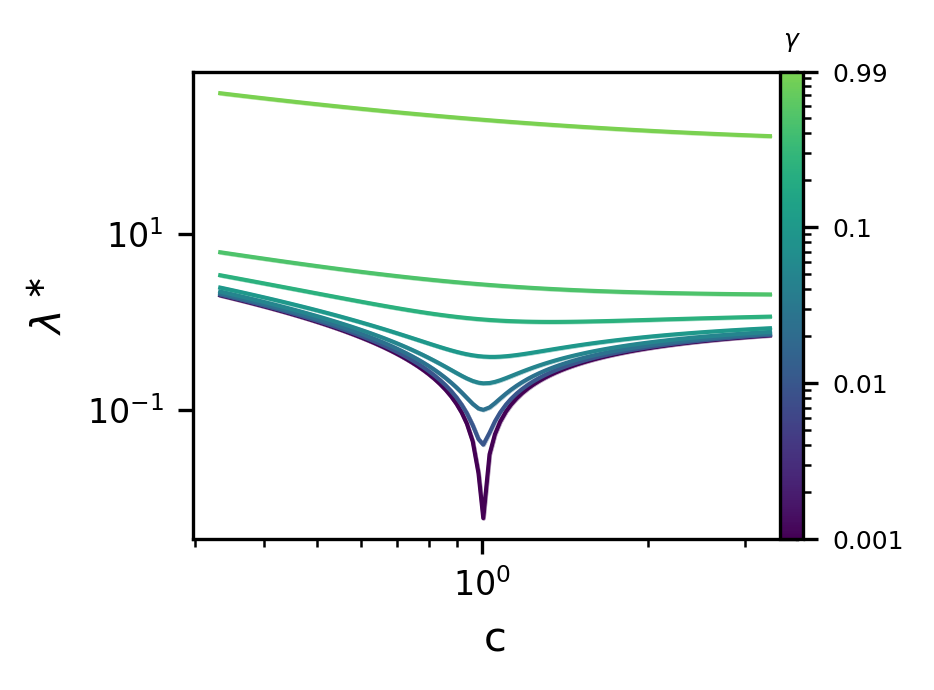}
  \captionof{figure}{\label{fig:LambdaOptLin}\centering Values of $\lambda^\ast$ versus $c$ varying $\gamma$ for the linear kernel}
\end{minipage}
\begin{minipage}{.48\textwidth}
  \centering
  \includegraphics[width=.9\textwidth]{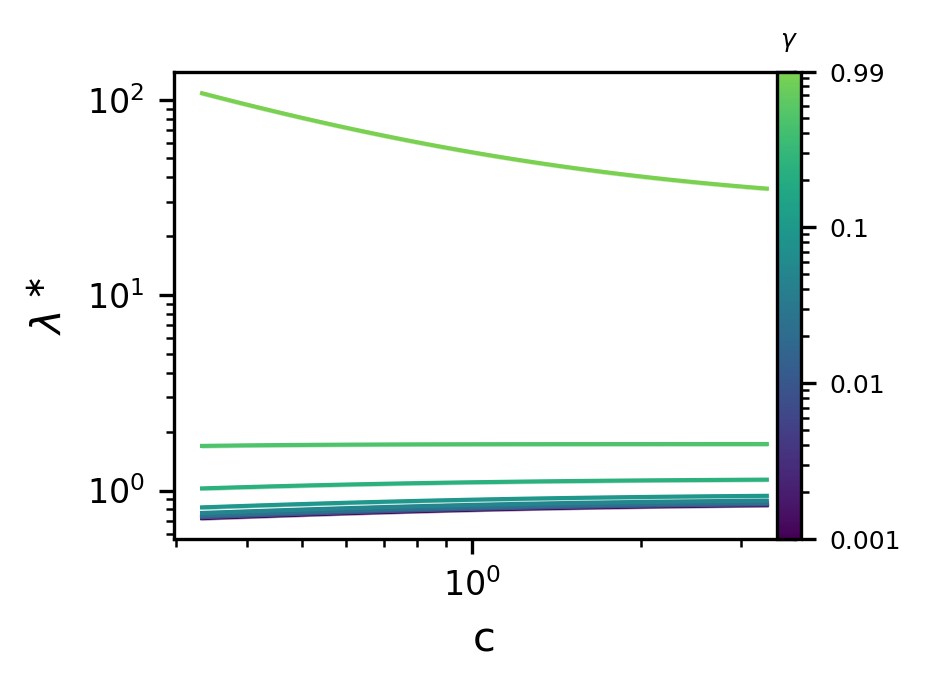}
  \captionof{figure}{\label{fig:LambdaOptGauss}\centering Values of $\lambda^\ast$ versus $c$ varying $\gamma$ for the Gaussian kernel}
\end{minipage}
\vspace{-.5cm}
\end{figure}

\paragraph{Real data without whitening.} To examine the effect that whitening has on the error curves, we reconsider the experiments producing Figures 3 (left; MNIST) and 9 (left; CIFAR10) where $X$ and $Y$ are only \emph{normalised}, that is, we subtract the sample means and divide by the sample deviation. The results are reported in Figure \ref{fig:NoWhitening}. As expected from \citet{couillet2011random} and the results of Figure \ref{fig:IllConditioned}, the curves resemble their whitened counterparts with some spurious ``bumps''.

\begin{figure}
    \centering
    \includegraphics[width=0.48\textwidth]{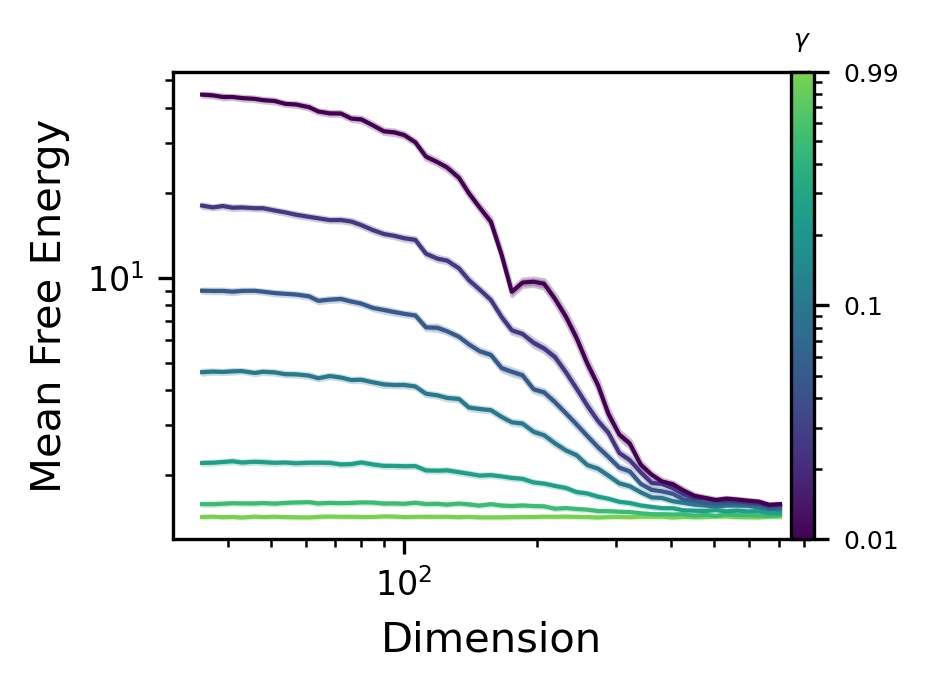}\includegraphics[width=0.48\textwidth]{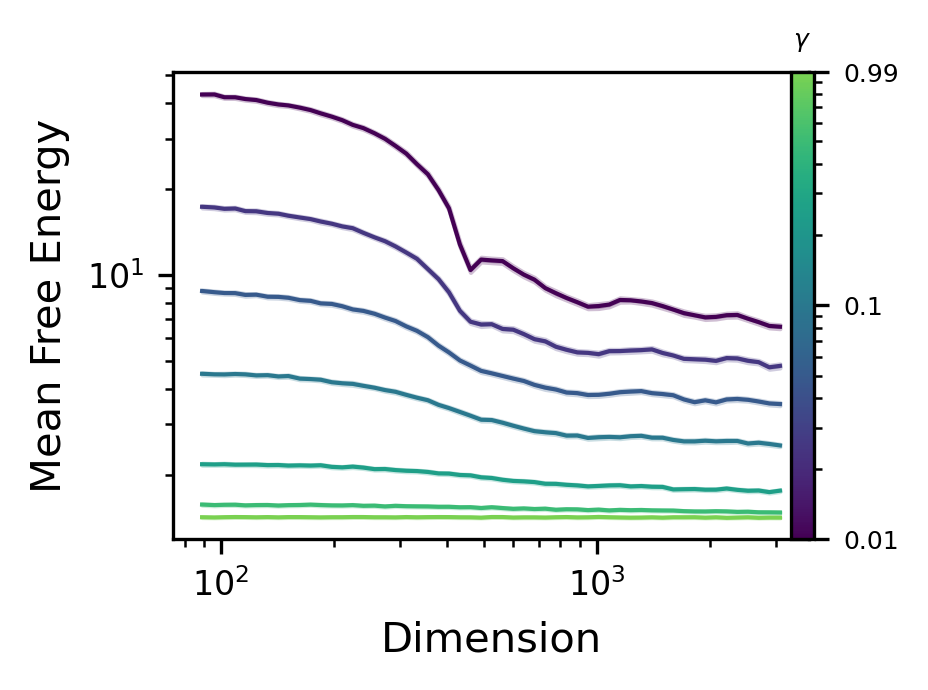}
    \caption{\centering Error curves for mean Bayes free energy for the MNIST (left) and CIFAR10 (right) datasets, with linear kernels; $\lambda = \lambda^\ast$.}
    \label{fig:NoWhitening}
\end{figure}

\section{Details of Experiments}
\label{sec:ExpDetails}

In each figure shown throughout this work, a performance metric has been calculated for varying dataset size $n$, input dimension $d$, and hyperparameters $\gamma,\lambda$. For experiments involving synthetic data, $X \in \mathbb{R}^{n\times d}$ has iid rows drawn from $\mathcal{N}(0,\Sigma)$, and $Y = (Y_i)_{i=1}^n$ is comprised of iid samples from $\mathcal{N}(0,\sigma^2)$ (where $\Sigma = I$ and $\sigma = 1$ unless specified otherwise). For PPL2 and PPNLL, the expectation is computed over iid scalar test points $x, y \sim \mathcal{N}(0,1)$. Runs are averaged over a number of iterations, and 95\% confidence intervals (under the central limit theorem) are highlighted. In Table 2 
we present the parameters used for each figure.

\begin{figure}\centering\includegraphics[height=\textwidth,angle=90]{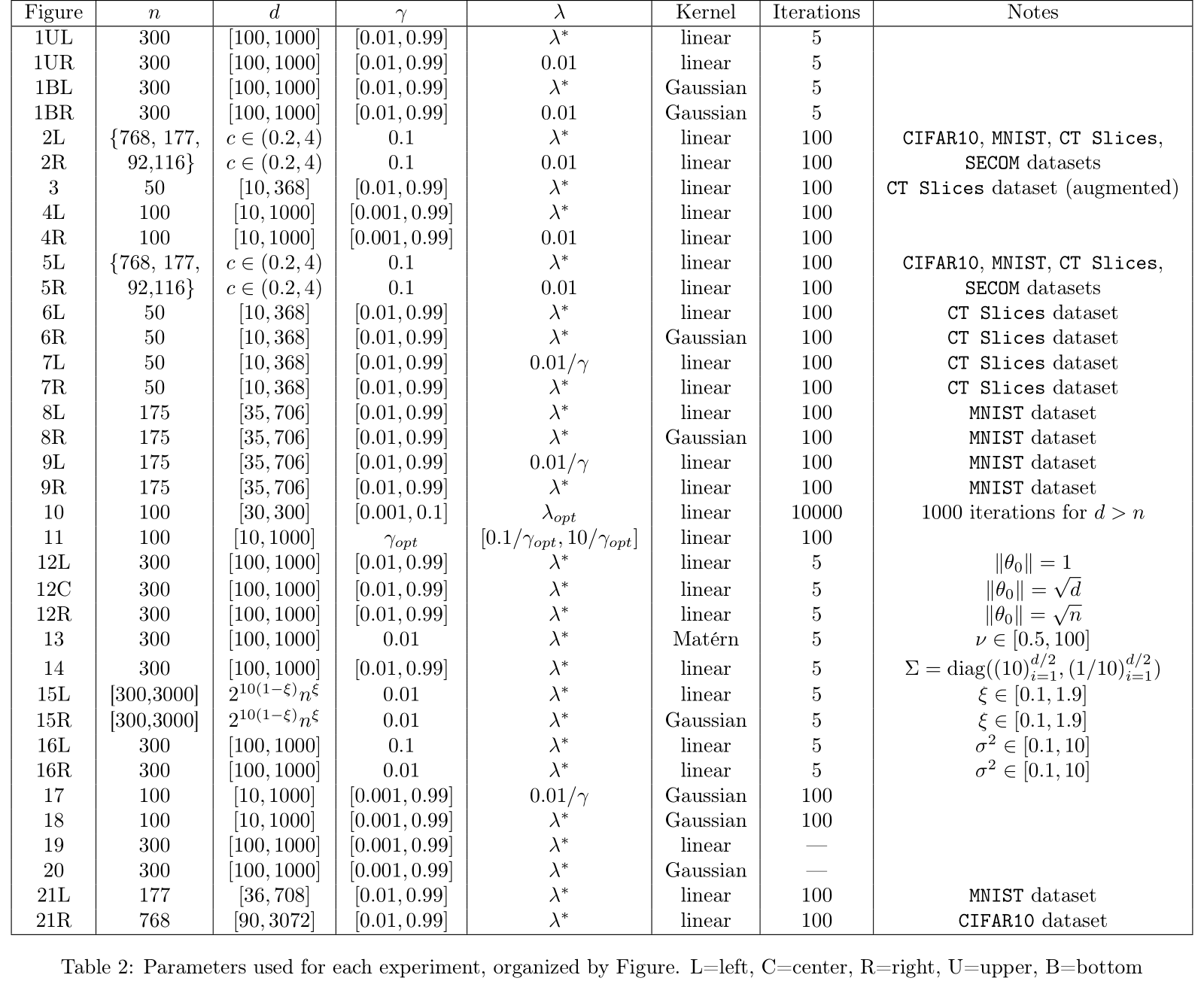}\end{figure}

\section{Normalized Data implies Small Prior Mispecification}
\label{sec:Whitening}

In Figure~\ref{fig:Misspec}, we explored the effect of changing the mean of the data-generating process from that of the prior. It was found that provided the mean of the data-generating process did not differ too significantly from that of the prior, the monotonicity of the error curves in Bayes free energy appeared unaffected. Here we show that when the data is normalized and whitened, and a zero-mean prior is chosen, the mean of the data-generating process will \emph{never} differ too significantly from the prior.

As above, assume that the labels satisfy $Y_i = f(X_i) + \epsilon_i$ for some $f:\mathbb{R}^d \to \mathbb{R}$ and zero-mean iid $\epsilon_i$. Now, we also assume that $Y$ has been normalized so that $\mbox{Var}(Y) = 1$. Similarly, we assume that $X$ has been normalized and whitened, so that it has zero mean and unit covariance. For simplicity of argument, assume further that $X_i$ are normal, that is, $X_i \overset{\text{iid}}{\sim} Z \sim \mathcal{N}(0,I)$. In the linear case where $f(x) = \theta \cdot x$, since 
\[
1 = \mbox{Var}(Y) \geq \mbox{Var}f(Z) = \|\theta\|^2,
\]
this implies that the magnitude of the components of $\theta$ are bounded on average by $d^{-1/2}$. This is the scenario seen in Figure~\ref{fig:Misspec}(left). Indeed, the scenarios in Figure~\ref{fig:Misspec}(center) and \ref{fig:Misspec}(right), which exhibit different error curves, satisfy $\mbox{Var}(Y) = n$ and $\mbox{Var}(Y) = d$ respectively, both of which are considerably larger than $1$.

The same principle holds for more general $f$. By a reverse Gaussian Poincar\'{e} inequality \citep[Proposition 3.5]{cacoullos1982upper},
\[
    1 = \mbox{Var}(Y) \geq \mbox{Var}f(Z) \geq \frac{1}{d} \left(\sum_{i=1}^d \mathbb{E}\partial_i f(Z)\right)^2,
\]
where $\partial_i$ denotes the $i$-th partial derivative. Therefore, the average coordinate-wise gradient of $f$, $\mathbb{E}\partial_I f(X)$ (where $I$ is uniform over $\{1,\dots,d\}$), is bounded above and below by
\[
-\sqrt\frac{1}{d}\leq \mathbb{E}\partial_I f(X) = \frac{1}{d} \sum_{i=1}^d \mathbb{E}\partial_i f(X) \leq \sqrt\frac{1}{d}.
\]

\section{Digamma Function}

Before treating the random matrix theory, we will need some auxiliary results concerning the \emph{digamma function}. Let $\Gamma(z)$ be the Gamma function, defined for $z > 0$ by $\Gamma(z) = \int_0^\infty t^{z-1} e^{-t} \dd t$. The \emph{digamma function} $\psi(z)$ is the derivative of the logarithm of the Gamma function, that is $\psi(z) = \frac{\dd}{\dd z} \log \Gamma(z)$. The digamma function satisfies the following properties:
\begin{itemize}
    \item $\psi(z + 1) = \psi(z) + z^{-1}$ for any $z > 0$;
    \item as $z \to \infty$, $\psi(z) / \log z \to 1$;
    \item letting $\eulermac = \psi(1)$ denote the Euler-Mascheroni constant, \[
    \psi(z + 1) = -\eulermac + \int_0^1 \left(\frac{1 - t^z}{1 - t}\right) \dd t.
    \]
\end{itemize}
The digamma function behaves well under summation. In particular, we have the following lemma.
\begin{lemma}
For any positive integer $n$ and any real number $z > -1$,
\[
\sum_{i=1}^{n}\psi(z+i)=(n+z)\psi(n+z)-z\psi(z)-n.
\]
\end{lemma}
\begin{proof}
From the integral representation for the digamma function:
\[
\psi(z)=-\eulermac+\int_{0}^{1}\left(\frac{1-t^{z-1}}{1-t}\right)\dd t,
\]
since $\sum_{i=1}^n t^{z+i-1} = t^z \sum_{i=0}^{n-1} t^i = t^z \frac{1 - t^n}{1 - t}$ for $0 \leq t < 1$,
\[
\sum_{i=1}^{n}\psi(z+i)=-n\eulermac+\int_{0}^{1}\frac{n(1-t)-t^{z}(1-t^{n})}{(1-t)^{2}}\dd t.
\]
Focusing on the integral term, note that by letting $f(t) = n(1-t) - t^z (1 - t^n)$ and $g(t) = (1-t)^{-1}$, since $g'(t) = (1-t)^{-2}$,
\begin{align*}
\int_{0}^{1}\frac{n(1-t)-t^{z}(1-t^{n})}{(1-t)^{2}}\dd t	&=\int_{0}^{1}f(t)g'(t)\dd t\\
&=\lim_{t\to 1^{-}}f(t)g(t)-f(0)g(0)-\int_{0}^{1}f'(t)g(t)\dd t.
\end{align*}
Since $\lim_{t\to1^{-}} (1 - t^n)/(1-t) = n$, $\lim_{t \to 1^{-}} f(t)g(t) = 0$, and so
\begin{align*}
\int_{0}^{1}\frac{n(1-t)-t^{z}(1-t^{n})}{(1-t)^{2}}\dd t
&=-n-\int_{0}^{1}\frac{-n-zt^{z-1}+(n+z)t^{n+z-1}}{1-t}\dd t\\
&=-n-\int_{0}^{1}\frac{(n+z)(t^{n+z-1}-1)-z(t^{z-1}-1)}{1-t}\dd t\\
&=-n+(n+z)\int_{0}^{1}\frac{1-t^{n+z-1}}{1-t}\dd t-z\int_{0}^{1}\frac{1-t^{z-1}}{1-t}\dd t\\
&=-n+(n+z)\left[\psi(n+z)+\gamma_{\text{EM}}\right]-z\left[\psi(z)+\eulermac\right]\\
&=-n+n\eulermac+(n+z)\psi(n+z)-z\psi(z).
\end{align*}
The result immediately follows
\end{proof}
Using this lemma, we can obtain an explicit expression for the sum of digamma functions with increment $\frac12$. This will be particularly useful for computing determinants of Wishart matrices. 
\begin{lemma}
For any positive integers $n$ and $d$ with $n > d$,
\begin{align*}
\sum_{i=1}^{d}\psi\left(\frac{n-i+1}{2}\right)	&=\frac{n}{2}\psi\left(\frac{n}{2}\right)-\left(\frac{n-d}{2}\right)\psi\left(\frac{n-d}{2}\right)-d\\
&\qquad+\left(\frac{n-1}{2}\right)\psi\left(\frac{n-1}{2}\right)\\
&\qquad-\left(\frac{n-d-1}{2}\right)\psi\left(\frac{n-d-1}{2}\right).
\end{align*}
\end{lemma}
\begin{proof}
First, note that
\[
\sum_{i=1}^{d}\psi\left(\frac{n-i+1}{2}\right)=\sum_{i=1}^{d}\psi\left(\frac{n-d+i}{2}\right).
\]
We consider the cases where $d$ is even and odd separately. When $d$ is even,
\begin{align*}
\sum_{i=1}^{d}\psi\left(\frac{n-d}{2}+\frac{i}{2}\right)
&=\sum_{i=1}^{d/2}\psi\left(\frac{n-d}{2}+i\right)+\psi\left(\frac{n-d}{2}+i-\frac{1}{2}\right)\\
&=\frac{n}{2}\psi\left(\frac{n}{2}\right)-\left(\frac{n-d}{2}\right)\psi\left(\frac{n-d}{2}\right)-d\\
&\qquad+\left(\frac{n-1}{2}\right)\psi\left(\frac{n-1}{2}\right)\\
&\qquad-\left(\frac{n-d-1}{2}\right)\psi\left(\frac{n-d-1}{2}\right).
\end{align*}
Now assume that $d$ is odd. Then
\begin{align*}
\sum_{i=1}^{d}\psi\left(\frac{n-d}{2}+\frac{i}{2}\right)
&=\psi\left(\frac{n}{2}\right)+\sum_{i=1}^{d-1}\psi\left(\frac{(n-1)-(d-1)}{2}+\frac{i}{2}\right)\\
&=\psi\left(\frac{n}{2}\right)+\frac{n-1}{2}\psi\left(\frac{n-1}{2}\right)-\left(\frac{n-d}{2}\right)\psi\left(\frac{n-d}{2}\right)-d+1\\
&\qquad+\left(\frac{n}{2}-1\right)\psi\left(\frac{n}{2}-1\right)-\left(\frac{n-d-1}{2}\right)\psi\left(\frac{n-d-1}{2}\right).
\end{align*}
But now, since $z\psi(z+1) = z\psi(z)+1$, $(\frac{n}2-1)\psi(\frac{n}2-1) = (\frac{n}2-1)\psi(\frac{n}2)-1$, and so
\[
\psi\left(\frac{n}{2}\right)+\left(\frac{n}{2}-1\right)\psi\left(\frac{n}{2}-1\right)=\frac{n}{2}\psi\left(\frac{n}{2}\right)-1.
\]
The result now follows.

\end{proof}

\section{Marchenko-Pastur Theory}

In this section, we prove several lemmas concerning limiting traces and log-determinants of Wishart matrices that will prove foundational for proving our main results. The fundamental theorem in this section is the Marchenko-Pastur Theorem, which describes the limiting spectral distribution of Wishart matrices. The following can be obtained from pg. 51 of \citet{couillet2011random}. 

\begin{theorem}[Marchenko-Pastur Theorem]
\label{thm:MP}
For each $n=1,2,\dots$, let $X_n \in \mathbb{R}^{n\times d}$ be a matrix of iid random variables with zero mean and unit variance. If $n, d \to \infty$ with $d / n \to c \in (0,\infty)$, then for every $z \in \mathbb{C}\backslash\{0\}$,
\[
d^{-1}\mathbb{E}\tr((n^{-1}X_n^\top X_n - zI)^{-1}) \to m(z) \coloneqq \frac{1 - c - z - \sqrt{(z-c-1)^2-4c}}{2 c z},
\]
noting that $m(z)$ satisfies $m = 1/(1-c-z-czm)$.
\end{theorem}

For the remainder of this section, we assume the conditions of Theorem \ref{thm:MP}, that is, for each $n=1,2,\dots$, we let $X_n \in \mathbb{R}^{n\times d}$ be a matrix of iid random variables with zero mean and unit~variance.

\begin{lemma}[Trace of Inverse Matrix]
\label{lem:TraceInv}
Let $n, d \to \infty$ with $d/n \to c \in (0,1]$ and assume that $\mu_n$ is a sequence of real numbers such that $\mu_n \to \mu \in (0,\infty)$ as $n \to \infty$. Then
\[
n^{-1} \mathbb{E}\tr((d^{-1} X_n^\top X_n + \mu_n I)^{-1}) \to T(\mu,c) \coloneqq \frac{c-1-c\mu+\sqrt{(c\mu+c+1)^{2}-4c}}{2\mu},
\]
and $T(\mu,c)$ satisfies $T = c^2 / (1 - c + c\mu + \mu T)$. Similarly, if $d/n \to c \in (1,\infty)$, then
\[
n^{-1} \mathbb{E}\tr((d^{-1} X_n X_n^\top + \mu_n I)^{-1}) \to \tilde{T}(\mu,c) \coloneqq \frac{1-c-c\mu+\sqrt{(c\mu+c+1)^{2}-4c}}{2\mu},
\]
and $\tilde{T}(\mu,c)$ satisfies $\tilde{T}= c / (c-1+c\mu+\mu\tilde{T})$ and $\tilde{T}(\mu,c) = c^2 T(c\mu, c^{-1})$. 
\end{lemma}
\begin{proof}
By the Neumann series, $(A+\epsilon I)^{-1} = A^{-1} + \mathcal{O}(\epsilon)$ as $\epsilon \to 0^+$. Therefore,
\begin{align*}
n^{-1}\mathbb{E}\tr\left((d^{-1}X_n^{\top}X_n+\mu_{n}I)^{-1}\right)
    &=n^{-1}\mathbb{E}\tr\left(\left(d^{-1}X_n^{\top}X_n+\mu I\right)^{-1}\right)+o(1)\\
    &=n^{-1}\mathbb{E}\tr\left(\left(\frac nd n^{-1}X_n^{\top}X_n+\mu I\right)^{-1}\right)+o(1)\\
    &=\frac dn n^{-1}\mathbb{E}\tr\left(\left(n^{-1}X_n^{\top}X_n+\frac dn \mu I\right)^{-1}\right)+o(1).\\
\end{align*}
By the Marchenko-Pastur Theorem, letting $z = -c\mu$,
\begin{align*}
\frac dn n^{-1}\mathbb{E}\tr\left((n^{-1}X_n^{\top}X_n+c\mu I)^{-1}\right)
    &=\frac{d^2}{n^2}d^{-1}\mathbb{E}\tr\left((n^{-1}X_n^{\top}X_n+c\mu I)^{-1}\right)\\
    &=\frac{d^2}{n^2}d^{-1}\mathbb{E}\tr\left((n^{-1}X_n^{\top}X_n-zI)^{-1}\right)\\
    &\to c^2\cdot\frac{1-c-z-\sqrt{(z-c-1)^{2}-4c}}{2cz}\\
    &=\frac{c-1-c\mu+\sqrt{(c\mu+c+1)^{2}-4c}}{2\mu}.
\end{align*}
On the other hand, when $c > 1$, letting $\tilde{X}_n = X_n^\top \in \mathbb{R}^{d \times n}$, the Marchenko-Pastur Theorem immediately implies
\begin{align*}
n^{-1}\mathbb{E}\tr\left((d^{-1}X_n X_n^{\top}+\mu_{n}I)^{-1}\right)	&=n^{-1}\mathbb{E}\mbox{tr}((d^{-1}\tilde{X}_n^\top \tilde{X}_n+\mu I)^{-1})+o(1)\\
&\to \frac{c^{-1} - 1 - \mu + \sqrt{(\mu + c^{-1} + 1)^2 - 4 c^{-1}}}{2 c^{-1} \mu} = \tilde{T}(\mu,c).
\end{align*}
\end{proof}

Now we turn our attention to the log-determinant, which also depends exclusively on the spectrum. Our method of proof relies on Jacobi's formula, which relates the log-determinant to the trace of the matrix
inverse.

\begin{lemma}[Log-Determinant]
\label{lem:Deter}
Let $n,d \to \infty$ such that $d/n \to c \in (0,1]$ and assume that $\mu_n$ is a sequence of real numbers such that $\mu_n \to \mu \in (0,\infty)$ as $n \to \infty$. Then \[
\frac{1}{n}\mathbb{E}\log\det(d^{-1}X_n^\top X_n + \mu_n I) \to D(\mu,c),\] where
\begin{align*}
D(\mu,c) & \coloneqq (c-1)\log(1-c)-c\log c-c+\int_0^\mu T(t,c)\dd t\\
&=\log\left(1 + \frac{T(\mu,c)}{c}\right)-\frac{T(\mu,c)}{c+T(\mu,c)}-c\log\left(\frac{T(\mu,c)}c\right).
\end{align*}
Similarly, if $d/n \to c \in (1,\infty)$, then \[\frac{1}{n}\mathbb{E}\log\det(d^{-1}X_n X_n^\top + \mu_n I) \to \tilde{D}(\mu,c),\] where
\begin{align*}
\tilde{D}(\mu,c) &\coloneqq (1-c)\log(c-1)+(c-1)\log c-1+\int_{0}^{\mu}\tilde{T}(t,c)\dd t,\\
&=c\log\left(1+\frac{\tilde{T}(\mu,c)}{c}\right)-\frac{c\tilde{T}(\mu,c)}{c+\tilde{T}(\mu,c)}-\log\tilde{T}(\mu,c).
\end{align*}
\end{lemma}
\begin{proof}
By Jacobi's formula and Taylor's theorem, $\log \det (A + \epsilon I) = \log \det A + \mathcal{O}(\epsilon)$ as $\epsilon \to 0^+$, and so
\[
n^{-1}\mathbb{E}\log\det(d^{-1}X_n^\top X_n + \mu_n I) = n^{-1} \mathbb{E}\log\det(d^{-1}X_n^\top X_n + \mu I) + o(1).
\]
Furthermore,
\begin{align*}
\frac1n \mathbb{E}\log\det(d^{-1}X_n^\top X_n + \mu I) 
    &= \frac1n \mathbb{E}\log\det(d^{-1}X_n^\top X_n) + \frac1n\int_0^{\mu} \mathbb{E}\tr\left((d^{-1}X_n^\top X_n + t I)^{-1}\right) \dd t,\\
    &= \frac1n \mathbb{E}\log\det(d^{-1}X_n^\top X_n) + \int_0^{\mu} T(t,c) \dd t + o(1),
\end{align*}
and so it suffices to consider the case $\mu = 0$. Since the log-determinant depends only on the spectrum of $X_n$, and the spectrum of $n^{-1} X_n^\top X_n$ is asymptotically equivalent to that of $n^{-1} W_n^\top W_n$, where $W_n$ is a Wishart-distributed matrix, it will suffice to consider the limit of $n^{-1}\mathbb{E}\log \det(d^{-1}W_n^\top W_n)$. First, recall that \citep[B.81]{bishop2006pattern}
\begin{align*}
\mathbb{E}\log\det(W_n^\top W_n) &= d \log 2 + \sum_{i=1}^d \psi\left(\frac{n-i+1}{2}\right)\\
&= d \log 2 + n\psi\left(\frac{n}{2}\right) - (n-d)\psi\left(\frac{n-d}{2}\right) \\
&+ \mathcal{O}(n^{-1}) + \mathcal{O}(d^{-1}). 
\end{align*}
Since $\psi(x) = \log x + \mathcal{O}(x^{-1})$, letting $d = [cn]$, there is
\begin{align*}
\mathbb{E}\log\det\left(W_n^{\top}W_n\right)&\sim d\log2+n\log\left(\frac{n}{2}\right)-(n-d)\log\left(\frac{n-d}{2}\right)-d\\
&\sim n\log n-(n-cn)\log\left(n-cn\right)-cn\\
&\sim n\log n-(1-c)n\log n-(1-c)n\log(1-c)-cn\\
&\sim cn\log n-(1-c)n\log(1-c)-cn.
\end{align*}
Therefore,
\begin{align*}
n^{-1}\mathbb{E}\log\det(d^{-1}W_n^{\top}W_n)&\sim\frac{cn\log n-(1-c)n\log(1-c)-cn-cn\log cn}{n}\\
&\to(c-1)\log(1-c)-c - c\log c,
\end{align*}
and so
\[
\frac1n \mathbb{E}\log\det(d^{-1}X_n^\top X_n + \mu_n I) \to D(\mu,c) \coloneqq (c-1)\log(1-c)-c\log c-c+\int_0^\mu T(t,c)\dd t.
\]
To obtain the second equality, we will need to compute the integral term. 
First, observe that by a change of variables, $\int_0^\mu T(t,c) \dd t = \int_0^{c\mu} \tau(t,c) \dd t$, where
\[
\tau(t, c) = \frac{c-1-t+\sqrt{(t+c+1)^{2}-4c}}{2t},
\]
and $T(t, c) = c\tau(c\mu, c)$.  Observe that we can rewrite $\tau$ as
\begin{align*}
\tau(t,c) &= \frac{(c+1+t)^{2}-4c-(t+1-c)^{2}}{2t\left[\sqrt{(c+1+t)^{2}-4c}+(t+1-c)\right]} \\
&=\frac{2c}{\sqrt{(c+1+t)^{2}-4c}+(t+1-c)}.
\end{align*}
Now, let 
\[
v = v(t) = \frac{c+1+t+\sqrt{(c+1+t)^2-4c}}{2},
\]
so that $\tau(t,c) = 2c / (2v - 2c) = c / (v - c)$. Note that $v^2 - (c+t+1)v + c = 0$. Differentiating this relation in $t$, we find
\[
2 v v' - v - (c + t + 1)v'  = 0,
\]
where $v' = \dd v / \dd t$, and hence
\[
v' = \frac{v}{2v - (c+t+1)}.
\]
But since $v^2 + c = (c+t+1)v$, 
\begin{align*}\label{eqn:vprime}
v' = \frac{v}{2v - \frac{v^2+c}{v}} = \frac{v^2}{v^2 - c}.
\end{align*}

Altogether, 
\[
\int \tau(t,c) \dd t = \int \frac{c(v^2 - c)}{(v-c)v^2} \dd v.
\]
From a partial fraction expansion, 
\[
\frac{c(v^2 - c)}{(v-c)v^2} = \frac{A}{v^2} + \frac{B}{v} + \frac{C}{v-c},
\]
we find that $c(v^2-c) = A(v-c)+Bv^2 - cBv + Cv^2$, implying that $B+C=c$, $A-cB =0$ and $-Ac = -c^2$. Therefore, $A = c$, $B = 1$, and $C = c-1$, so
\[
\frac{c(v^2-c)}{(v-c)v^2} = \frac{c}{v^2} + \frac1v+\frac{c-1}{v-c}.
\]
Hence, an antiderivative of $\tau$ is given by
\[
-\frac{c}{v} + \log v + (c-1)\log(v-c).
\]
Since $v \to 1$ as $t \to 0$,
\[
\int_0^{c\mu} \tau(t,c)\dd t = -\frac{c}{v} + \log v + (c-1)\log(v-c) + c - (c-1)\log(1-c).
\]
Finally, since $v = c(1+\tau(c\mu,c))/\tau(c\mu,c) = c(c+T(\mu,c))/T(\mu,c)$, the result for $n > d$ follows. 

Now we consider the $d > n$ case. Then we have
\begin{align*}
n^{-1}\mathbb{E}\log\det\left(d^{-1}X_{n}X_{n}^{\top}+\mu_{n}I\right)	&=n^{-1}\mathbb{E}\log\det\left(d^{-1}X_{n}X_{n}^{\top}+\mu I\right)+o(1)\\
&=\frac{d}{n}d^{-1}\mathbb{E}\log\det\left(\frac{n}{d}n^{-1}X_{n}X_{n}^{\top}+\mu I\right)+o(1)\\
&=\frac{d}{n}d^{-1}\mathbb{E}\log\det\left(n^{-1}X_{n}X_{n}^{\top}+\frac{d}{n}\mu I\right)+\log\left(\frac{n}{d}\right)+o(1)\\
&=cd^{-1}\mathbb{E}\log\det\left(n^{-1}X_{n}X_{n}^{\top}+c\mu I\right)-\log c+o(1)\\
&\to cD(c\mu,c^{-1})-\log c.
\end{align*}
From the first expression for $D(\mu,c)$, there is
\begin{align*}
cD(c\mu,c^{-1})&=c(c^{-1}-1)\log(1-c^{-1})-\log c^{-1}-1+\int_{0}^{c\mu}cT(t,c^{-1})\dd t\\
&= (1-c)\log(c-1)+c\log c-1+\int_{0}^{\mu}c^{2}T(ct,c^{-1})\dd t\\
&= (1-c)\log(c-1)+c\log c-1+\int_{0}^{\mu}\tilde{T}(t,c)\dd t.
\end{align*}
Finally, from the second expression for $D(\mu,c)$,
\begin{align*}
cD(c\mu,c^{-1})	%
&=c\log\left(1+\frac{T(c\mu,c^{-1})}{c^{-1}}\right)-\frac{cT(c\mu,c^{-1})}{c^{-1}+T(c\mu,c^{-1})}-\log\left(\frac{T(c\mu,c^{-1})}{c^{-1}}\right),\\
&=c\log\left(1+\frac{\tilde{T}(\mu,c)}{c}\right)-\frac{c\tilde{T}(\mu,c)}{c+\tilde{T}(\mu,c)}-\log\left(\frac{\tilde{T}(\mu,c)}{c}\right),
\end{align*}
from which the result follows. %

\end{proof}

\section{Kernels and Gram Matrices}

To extend the results of the previous section to Gram matrices, we rely on the approximation theory developed in \citet{karoui2010}. For a continuous function $\kappa:\mathbb{R} \to \mathbb{R}$ that is continuously differentiable on $(0,\infty)$, two types of kernels are considered:
\begin{enumerate}[label=(\Roman*)]
\item \textbf{Inner product kernels:} $k(x, y) = \kappa(x^\top y / d)$ for $x, y\in \mathbb{R}^d$, and $\kappa$ is three-times continuously differentiable in a neighbourhood of zero with $\kappa'(0) > 0$.
\item \textbf{Radial basis kernels:} $k(x, y) = \kappa(\|x-y\|^2 / d)$ for $x, y \in \mathbb{R}^d$, and $\kappa$ is three-times continuously differentiable on $(0,\infty)$ with $\kappa' < 0$.
\end{enumerate}
Let $\|A\|_2$ denote the spectral norm of a matrix $A$. The following theorem combines Theorems 2.1 and 2.2 in \citet{karoui2010}.
\begin{theorem}
\label{thm:ElKaroui}
For each $n=1,2,\dots$, let $X_n^1,\dots,X_n^n$ be independent and identically distributed zero-mean random vectors in $\mathbb{R}^d$ with $\cov(X_k^i) = \sigma^2 I$ and $\mathbb{E}\|X_k^i\|^{5+\delta} < \infty$ for some $\delta > 0$. For a kernel $k$ of type (I) or (II), consider the Gram matrices $K_X^n \in \mathbb{R}^{n\times n}$ with entries $(K_X^n)_{ij} = k(X_n^i, X_n^j)$. If $n, d \to \infty$ such that $d / n \to c \in (0,\infty)$, then there exists an integer $k$ and a bounded sequence of rank $k$ matrices $C_1,C_2,\dots$ such that
\[
\|K_X^n - (\alpha d^{-1} XX^\top + \beta I + C_n)\|_2 \to 0,
\]
where the constants $\alpha,\beta$ for cases (I) and (II) are, respectively,
\begin{enumerate}[label=(\Roman*)]
\item \textbf{Inner product kernels:} $\alpha = \kappa'(0)$, $\beta = \kappa(\sigma^2) - \kappa(0) - \kappa'(0)\sigma^2$;
\item \textbf{Radial basis kernels:} $\alpha = -2\kappa'(2\sigma^2)$, $\beta = \kappa(0) + 2\sigma^2 \kappa'(2\sigma^2) - \kappa(2\sigma^2)$.
\end{enumerate}
\end{theorem}

For the remainder of this section, we assume the hypotheses of Theorem \ref{thm:ElKaroui}, so that $\|K_X^n - (\alpha d^{-1} XX^\top + \beta I)\|_2 \to 0$ for some appropriate $\alpha > 0$ and $\beta \in \mathbb{R}$. To apply Theorem \ref{thm:ElKaroui} with the results of the previous section, we require the following basic lemma.
\begin{lemma}
\label{lem:TraceIneq}
For any symmetric positive-definite matrices $A, B \in \mathbb{R}^{n\times n}$ and $v > 0$,
\begin{align*}
\frac1n |\tr((A + vI)^{-1}) - \tr((B+vI)^{-1})| &\leq \frac{\|A-B\|_2}{v^2} \\
\frac1n |\log \det (A + vI) - \log \det (B + vI)| &\leq \frac{\|A-B\|_2}{v}.
\end{align*}
\end{lemma}
\begin{proof}
Let $\lambda_1(A) \geq \cdots \geq \lambda_n(A)$ and $\lambda_1(B) \geq \cdots \geq \lambda_n(B)$ denote the eigenvalues of $A$ and $B$, respectively, in decreasing order. Recall from Weyl's perturbation theorem (see Corollary III.2.6 of \citet{bhatia2013matrix}) that $\max_{i=1,\dots,n} |\lambda_i(A) - \lambda_i(B)| \leq \|A - B\|_2$. By the Mean Value Theorem, for any $x,y>0$, $|(x+v)^{-1}-(y+v)^{-1}|\leq v^{-2}|x-y|$. Therefore,
\begin{align*}
\frac1n |\tr((A + vI)^{-1}) - \tr((B+vI)^{-1})| &= \frac1n \left|\sum_{i=1}^n \frac{1}{\lambda_i(A) + v} - \frac{1}{\lambda_i(B) + v}\right| \\
&\leq \frac{1}{v^2} \max_{i=1,\dots,n} |\lambda_i(A)-\lambda_i(B)| \\
&\leq \frac{1}{v^2} \|A - B\|_2.
\end{align*}
Similarly, the Mean Value Theorem implies that for any $x,y>0$, $|\log(x+v)-\log(y+v)|\leq v^{-1}|x-y|$, and so
\begin{align*}
\frac1n |\log\det(A + vI) - \log\det(B+vI)| &= \frac1n \left|\sum_{i=1}^n \log(\lambda_i(A) + v) - \log(\lambda_i(B) + v)\right| \\
&\leq \frac{1}{v} \max_{i=1,\dots,n} |\lambda_i(A)-\lambda_i(B)| \\
&\leq \frac{1}{v} \|A - B\|_2.
\end{align*}
\end{proof}

Combining Theorem \ref{thm:ElKaroui} and Lemma \ref{lem:TraceIneq} with Lemmas \ref{lem:TraceInv} and \ref{lem:Deter} yields the following corollary.
\begin{corollary}
\label{cor:ElKarouiLim}
Under the assumptions of Theorem \ref{thm:ElKaroui}, if $\mu_n$ is a sequence of positive real numbers such that $\mu_n \to \mu \in (0,\infty)$ as $n \to \infty$, then
\begin{align*}
\frac1n \mathbb{E}\tr((K_X^n + \mu_n I)^{-1}) &\to \begin{cases}
\frac{1-c}{\beta+\mu}+\frac{1}{\alpha}T\left(\frac{\beta+\mu}{\alpha},c\right) &\text{ if } c < 1 \\
\frac1{\alpha} \tilde{T}\left(\frac{\beta+\mu}{\alpha}, c\right) &\text{ if } c > 1,
\end{cases} \\
\frac1n \mathbb{E}\log \det(K_X^n + \mu_n I) &\to \begin{cases}
D\left(\frac{\beta+\mu}{\alpha},c\right)+\left(1-c\right)\log\left(\frac{\beta+\mu}{\alpha}\right)+\log\alpha &\text{ if } c < 1\\
\tilde{D}\left(\frac{\beta+\mu}{\alpha}, c\right) + \log \alpha & \text{ if } c > 1.
\end{cases}
\end{align*}
\end{corollary}
\begin{proof}
First consider the $c > 1$ case. Combining Theorem \ref{thm:ElKaroui} and Lemma \ref{lem:TraceIneq}, and noting that finite rank perturbations do not affect the limiting spectrum \citep[Lemma 2.1]{karoui2010}, we find that
\begin{align*}
\frac{1}{n}\mathbb{E}\mbox{tr}\left(K_{X}^{n}+\mu_{n}I\right)^{-1}	&=\frac{1}{n}\mathbb{E}\mbox{tr}\left(\alpha d^{-1}XX^{\top}+\beta I+\mu_{n}I\right)^{-1}+o(1)\\
&=\frac{1}{\alpha n}\mathbb{E}\mbox{tr}\left(d^{-1}XX^{\top}+\frac{\beta+\mu}{\alpha}I\right)^{-1}+o(1)\\
&\to\frac{1}{\alpha}\tilde{T}\left(\frac{\beta+\mu}{\alpha},c\right).
\end{align*}
Similarly, since $XX^\top \in \mathbb{R}^{n\times n}$,
\begin{align*}
\frac{1}{n}\mathbb{E}\log\det\left(K_{X}^{n}+\mu_{n}I\right)	&=\frac{1}{n}\mathbb{E}\log\det\left(\alpha d^{-1}XX^{\top}+\beta I+\mu I\right)+o(1)\\
&=\frac{1}{n}\mathbb{E}\log\det\left(d^{-1}XX^{\top}+\frac{\beta+\mu}{\alpha}I\right)+\log\alpha+o(1)\\
&\to \tilde{D}\left(\frac{\beta+\mu}{\alpha},c\right)+\log\alpha.
\end{align*}
For the $c < 1$ case, from the Woodbury matrix identity \citep[B.1.2]{pozrikidis2014introduction},
\begin{align*}
\mbox{tr}\left((\eta_{1}XX^{\top}+\eta_{2}I)^{-1}\right)	&=\frac{n}{\eta_{2}}-\mbox{tr}\left(\frac{\eta_{1}}{\eta_{2}}X\left(\eta_{2}I+\eta_{1}X^{\top}X\right)^{-1}X^{\top}\right)\\
&=\frac{n}{\eta_{2}}-\frac{1}{\eta_{2}}\mbox{tr}\left(\left(\eta_{2}I+\eta_{1}X^{\top}X\right)^{-1}\eta_{1}X^{\top}X\right)\\
&=\frac{n}{\eta_{2}}-\frac{1}{\eta_{2}}\mbox{tr}\left(I-\eta_{2}\left(\eta_{2}I+\eta_{1}X^{\top}X\right)^{-1}\right)\\
&=\frac{n-d}{\eta_{2}}+\mbox{tr}\left((\eta_{1}X^{\top}X+\eta_{2}I)^{-1}\right).
\end{align*}
Therefore,
\begin{align*}
\frac{1}{n}\mathbb{E}\mbox{tr}\left(K_{X}^{n}+\mu_{n}I\right)^{-1}	&=\frac{1}{\alpha n}\mathbb{E}\mbox{tr}\left(d^{-1}XX^{\top}+\frac{\beta+\mu}{\alpha}I\right)^{-1}+o(1)\\
&=\frac{1-\frac{d}{n}}{\beta+\mu}+\frac{1}{\alpha n}\mathbb{E}\mbox{tr}\left(d^{-1}X^{\top}X+\frac{\beta+\mu}{\alpha}I\right)^{-1}+o(1)\\
&\to\frac{1-c}{\beta+\mu}+\frac{1}{\alpha}T\left(\frac{\beta+\mu}{\alpha},c\right).
\end{align*}
Finally, from Sylvester's determinant theorem \citep[B.1.15]{pozrikidis2014introduction},
\begin{align*}
\log\det(\eta_{1}XX^{\top}+\eta_{2}I)	&=\log\det\left(\frac{\eta_{1}}{\eta_{2}}XX^{\top}+I\right)+n\log\eta_{2}\\
&=\log\det\left(\frac{\eta_{1}}{\eta_{2}}X^{\top}X+I\right)+n\log\eta_{2}\\
&=\log\det\left(\eta_{1}X^{\top}X+\eta_{2}I\right)+(n-d)\log\eta_{2}.
\end{align*}
Therefore,
\begin{align*}
\frac{1}{n}\mathbb{E}\log\det\left(K_{X}^{n}+\mu_{n}I\right)	&=\frac{1}{n}\mathbb{E}\log\det\left(d^{-1}XX^{\top}+\frac{\beta+\mu}{\alpha}I\right)+\log\alpha+o(1)\\
&=\frac{1}{n}\mathbb{E}\log\det\left(d^{-1}X^{\top}X+\frac{\beta+\mu}{\alpha}I\right)\\
&\qquad\qquad+\left(1-\frac{d}{n}\right)\log\left(\frac{\beta+\mu}{\alpha}\right)+\log\alpha+o(1)\\
&\to D\left(\frac{\beta+\mu}{\alpha},c\right)+\left(1-c\right)\log\left(\frac{\beta+\mu}{\alpha}\right)+\log\alpha.
\end{align*}
\end{proof}

\section{Proofs of Main Results}
\label{sec:Proofs}

With the underlying random matrix theory in place, we can begin to prove our main result in Theorem \ref{thm:Main}. Throughout this section, we assume the conditions of Theorem \ref{thm:Main}, that is, we let $X_1,X_2,\dots$ be independent and identically distributed zero-mean random vectors in $\mathbb{R}^d$ with unit covariance, satisfying $\mathbb{E}\|X_i\|^{5+\delta} < +\infty$ for some $\delta > 0$. For each $n=1,2,\dots$, let
\[
\mathcal{F}_{n}^{\gamma}=\tfrac{1}{2}\lambda Y^{\top}(K_{X}+\lambda\gamma I)^{-1}Y+\tfrac{1}{2}\log\det(K_{X}+\lambda\gamma I)-\tfrac{n}{2}\log\left(\tfrac{\lambda}{2\pi}\right).
\]
where $K_X \in \mathbb{R}^{n\times n}$ satisfies $K_X^{ij} = k(X_i,X_j)$ and $Y = (Y_i)_{i=1}^n$, with each $Y_i \sim \mathcal{N}(0,1)$. 

\begin{proposition}[El Karoui-Marchenko-Pastur Limit of the Bayes Free Energy]
\label{prop:EntropyLimit}
Assuming that $n,d\to\infty$ such that $d / n \to c \in (0,\infty)$, there is $n^{-1}\mathbb{E}\mathcal{F}_n^\gamma \to \mathcal{F}_\infty^\gamma$ where for $c < 1$,
\[
\mathcal{F}_\infty^\gamma=\frac{\lambda}{2}\left(\frac{1-c}{\beta+\gamma\lambda}+\frac{1}{\alpha}T\left(\frac{\beta+\gamma\lambda}{\alpha},c\right)\right)-\frac{1}{2}\log\left(\frac{\lambda}{2\pi\alpha}\right)
+\frac{1}{2}D\left(\frac{\beta+\gamma\lambda}{\alpha},c\right)+\frac{1}{2}\left(1-c\right)\log\left(\frac{\beta+\gamma\lambda}{\alpha}\right),
\]
and for $c > 1$,
\[
\mathcal{F}_\infty^\gamma=\frac{\lambda}{2\alpha}\tilde{T}\left(\frac{\beta+\gamma\lambda}{\alpha},c\right)-\frac{1}{2}\log\left(\frac{\lambda}{2\pi\alpha}\right)+\frac{1}{2}\tilde{D}\left(\frac{\beta+\gamma\lambda}{\alpha},c\right).
\]
\end{proposition}
\begin{proof}
Recalling that $\mathbb{E}[Y^\top A Y] = \tr(A)$ for any $A \in \mathbb{R}^{n\times n}$, since $K_X$ is independent of $Y$, 
\[
\frac1n \mathbb{E}\mathcal{F}_n^\gamma = \frac{\lambda}{2n} \mathbb{E}\tr((K_X+\lambda \gamma I)^{-1}) + \frac1{2n} \mathbb{E}\log\det(K_X + \lambda \gamma I) - \frac{1}{2}\log\left(\frac{\lambda}{2\pi}\right).
\]
The result follows by a direct application of Corollary \ref{cor:ElKarouiLim}. 
\end{proof}

\begin{proposition}[Optimal Temperature in the Bayes Free Energy]
\label{prop:OptTempBayes}
Assume that $\lambda = \mu / \gamma$ for some fixed $\mu > 0$. The limiting Bayes free energy $\mathcal{F}_\infty^\gamma$ is minimized in $\gamma$ at
\[
\gamma^\ast = \frac{\mu}{2(\beta+\mu)} [1 - c - c(\tfrac{\beta+\mu}{\alpha}) + \sqrt{(c(\tfrac{\beta+\mu}{\alpha}) + c + 1)^2 -4c}].
\]
\end{proposition}
\begin{proof}
First consider the case $c < 1$. If $\lambda = \mu / \gamma$, then
\[
\mathcal{F}_\infty^\gamma = \frac{\mu}{2\gamma}\left(\frac{1-c}{\beta+\mu}+\frac1{\alpha} T\left(\frac{\beta+\mu}{\alpha},c\right)\right) - \frac12 \log\left(\frac{\mu}{2\pi\gamma\alpha}\right)+
\frac12 D\left(\frac{\beta+\mu}{\alpha}, c\right) + \frac12 (1 - c)\log\left(\frac{\beta+\mu}{\alpha}\right).
\]
Note that as $\gamma \to 0^+$ or $\gamma \to \infty$, $\mathcal{F}_\infty^\gamma$, so if there exists only one point $\gamma^\ast$ where that $\partial\mathcal{F}_\infty^\gamma / \partial \gamma = 0$, then by Fermat's Theorem, $\gamma^\ast$ is the unique global minimizer of $\mathcal{F}_\infty^\gamma$.  
For $\mu$ fixed, we may differentiate in $\gamma$ to find that
\[
\frac{\partial \mathcal{F}_\infty^\gamma}{\partial \gamma} = -\frac{\mu}{2\gamma^2}\left(\frac{1-c}{\beta+\mu} + \frac{1}{\alpha}T\left(\frac{\beta+\mu}{\alpha}, c\right)\right) + \frac1{2\gamma}.
\]
Solving $\partial \mathcal{F}_\infty^\gamma / \partial \gamma = 0$ for $\gamma$, the optimal
\[
\gamma^\ast = \mu\left(\frac{1-c}{\beta+\mu} + \frac1\alpha T\left(\frac{\beta+\mu}{\alpha}, c\right)\right).
\]
Simplifying,
\[
\frac{1-c}{\beta+\mu} + \frac1{\alpha} \frac{c - 1 - c(\frac{\beta+\mu}{\alpha}) + \sqrt{(c(\frac{\beta+\mu}{\alpha})+c+1)^2 - 4c}}{2(\frac{\beta+\mu}{\alpha})} =
\frac{1 - c - c(\frac{\beta+\mu}{\alpha}) + \sqrt{(c(\frac{\beta+\mu}{\alpha})+c+1)^2 - 4c}}{2(\beta+\mu)},
\]
which implies the result for $c < 1$. On the other hand, for $c > 1$,
\[
\mathcal{F}_\infty^\gamma = \frac{\mu}{2\gamma\alpha}\tilde{T}\left(\frac{\beta+\mu}{\alpha}, c\right) - \frac12 \log\left(\frac{\mu}{2\pi\alpha\gamma}\right) + \frac12 \tilde{D}\left(\frac{\beta+\mu}{\alpha}, c\right),
\]
and once again, as $\gamma\to 0^+$ or $\gamma \to \infty$, $\mathcal{F}_\infty^\gamma \to \infty$, so a unique critical point is the unique global minimizer of $\mathcal{F}_\infty^\gamma$. For $\mu$ fixed, we differentiate in $\gamma$ to find
\[
\frac{\partial \mathcal{F}_\infty^\gamma}{\partial \gamma} = -\frac{\mu}{2\gamma^2 \alpha}\tilde{T}\left(\frac{\beta+\mu}{\alpha},c\right) + \frac{1}{2\gamma}.
\]
Solving $\partial\mathcal{F}_\infty^\gamma / \partial \gamma = 0$ for $\gamma$, the optimal
\[
\gamma^\ast = \frac{\mu}{\alpha} \tilde{T}\left(\frac{\beta+\mu}{\alpha}, c\right).
\]
Simplifying,
\[
\frac1\alpha \tilde{T}\left(\frac{\beta+\mu}{\alpha}, c\right) = \frac{1 - c - c(\frac{\beta+\mu}{\alpha}) + \sqrt{(c(\frac{\beta+\mu}{\alpha}) + c + 1)^2 -4c}}{2(\beta + \mu)},
\]
which implies the result for $c > 1$.
\end{proof}

In the sequel, we assume that the kernel itself depends on $\lambda$ in such a way that $\beta = \beta_0 \lambda$ for some $0 < \beta_0 < 1$. Let $\gamma_0 = \gamma + \beta_0$ and $\mu = \lambda \gamma_0 / \alpha$. For $c < 1$, the limiting Bayes free energy satisfies
\begin{align*}
\mathcal{F}_\infty^\gamma &= \frac{1}{2\gamma_0}(1 - c + \mu T(\mu,c)) - \frac12 \log\left(\frac{\mu}{2\pi \gamma_0}\right) + \frac12 D(\mu, c) + \frac12 (1 - c) \log \mu,\\
&= \frac{1}{2\gamma_0}(1 - c + \mu T(\mu,c)) - \frac12 \log\left(\frac{1}{2\pi \gamma_0}\right) + \frac12 D(\mu, c) - \frac{c}2 \log \mu.
\end{align*}
and for $c > 1$,
\[
\mathcal{F}_\infty^\gamma = \frac{\mu}{2\gamma_0} \tilde{T}(\mu,c) - \frac12 \log\left(\frac{\mu}{2\pi\gamma_0}\right) + \frac12 \tilde{D}(\mu, c).
\]

\begin{proposition}[Optimal Regularization in the Bayes Free Energy]
\label{prop:OptRegBayes}
The limiting Bayes free energy $\mathcal{F}_\infty^\gamma$ is minimized in $\lambda$ at
\[
\lambda^\ast = \frac{\alpha[(c+1)\gamma_0+\sqrt{(c-1)^{2}+4c\gamma_0^{2}}]}{c(1-\gamma_0^{2})}.
\]
\end{proposition}
\begin{proof}
Since $\mathcal{F}_\infty^\gamma$ is smooth for $\lambda \in (0,\infty)$ (and therefore $\mu \in (0,\infty)$), Fermat's theorem implies that any optimal temperature $\lambda^\ast$ must be a critical point of $\mathcal{F}_\infty^\gamma$ in $(0,\infty)$. First, consider the case where $c < 1$. 
Differentiating $\mathcal{F}_\infty^\gamma$ with respect to $\mu$,
\[
\frac{\partial \mathcal{F}_\infty^\gamma}{\partial \mu} = \frac{1}{2\gamma_0} \frac{\partial}{\partial \mu}(\mu T(\mu,c)) - \frac{c}{2\mu} + \frac12 T(\mu, c).
\]
Letting $U(\mu,c) = \mu T(\mu, c)$ and $U' = \frac{\partial U}{\partial \mu}$,
\begin{equation}
\label{eq:OptTempEqn1}
\frac{\partial \mathcal{F}_\infty^\gamma}{\partial \mu} = \frac{1}{2\mu}\left(\frac{\mu}{\gamma_0} U' + U - c\right).
\end{equation}
Noting that
\[
U(\mu, c) = \frac{c-1-c\mu + \sqrt{(c\mu+c+1)^2 - 4c}}{2},
\]
and
\[
U' = -\frac{c}{2} +  \frac{c(c\mu+c+1)}{2\sqrt{(c\mu+c+1)^2 - 4c}}
= c\cdot\frac{c\mu+c+1-\sqrt{(c\mu+c+1)^2-4c}}{2\sqrt{(c\mu+c+1)^2 - 4c}},
\]
and so $U'\sqrt{(c\mu+c+1)^2 - 4c} = c \cdot (c - U)$. Therefore, substituting into (\ref{eq:OptTempEqn1}) reveals
\[
\frac{\partial \mathcal{F}_\infty^\gamma}{\partial \mu} = \frac{1}{2c\mu}\left(\frac{c\mu}{\gamma_0} - \sqrt{(c\mu+c+1)^2 - 4c} \right) U'.
\]
Since $U' > 0$, $\partial \mathcal{F}_\infty^\gamma / \partial \mu = 0$ if and only if
\begin{equation}
\label{eq:OptTempEqn}
\frac{c\mu}{\gamma_0} = \sqrt{(c\mu+c+1)^2 - 4c}.
\end{equation}
This occurs when
\begin{equation}
\label{eq:OptTempQuad}
c^2 (1 - \gamma_0^2)\mu^2 - 2 c \mu(c+1)\gamma_0^2 - (c-1)^2 \gamma_0^2 = 0.
\end{equation}
If $\gamma_0 \geq 1$, then no positive solutions exist for $\mu$. On the other hand, if $\gamma < 1$, then only one positive solution exists, and is given by
\begin{align*}
\mu^{\ast}&=\frac{2c(c+1)\gamma_0^{2}+\sqrt{4c^{2}(c+1)^{2}\gamma_0^{4}+4c^{2}(1-\gamma_0^{2})(c-1)^{2}\gamma_0^{2}}}{2c^{2}(1-\gamma_0^{2})}\\
&=\frac{2c(c+1)\gamma_0^{2}+2c\gamma_0\sqrt{[(c+1)^{2}-(c-1)^{2}]\gamma_0^{2}+(c-1)^{2}}}{2c^{2}(1-\gamma_0^{2})}\\
&=\frac{(c+1)\gamma_0^{2}+\gamma_0\sqrt{(c-1)^{2}+4c\gamma_0^{2}}}{c(1-\gamma_0^{2})}.
\end{align*}

Next, consider the case $c > 1$. Differentiating $\mathcal{F}_\infty^\gamma$ with respect to $\mu$, we seek
\[
\frac{\partial \mathcal{F}_\infty^\gamma}{\partial \mu} = \frac{1}{2\gamma} \frac{\partial}{\partial \mu}(\mu \tilde{T}(\mu, c)) + \frac{1}{2} \tilde{T}(\mu, c) - \frac{1}{2\mu} = 0,
\]
or, equivalently,
\begin{equation}
\label{eq:OptTempEqn2}
\frac{\mu}{\gamma} \frac{\partial}{\partial \mu}(\mu \tilde{T}(\mu,c)) + \mu \tilde{T}(\mu,c) - 1 = 0.
\end{equation}
Letting $\tilde{U} = \mu \tilde{T}$ and $\tilde{U}' = \frac{\partial \tilde{U}}{\partial \mu}$, we require $\frac{\mu}{\gamma} U' + U - 1 = 0$. But since
\[
\tilde{U} = \frac{1 - c - c\mu + \sqrt{(c\mu+c+1)^2 - 4c}}{2},
\]
and
\[
\tilde{U}' = \frac{\partial \tilde{U}}{\partial \mu} = \frac{c(c+c\mu+1-\sqrt{(c\mu+c+1)^2-4c})}{2\sqrt{(c\mu+c+1)^2-4c}},
\]
we find that $\tilde{U}' \sqrt{(c\mu+c+1)^2 - 4c} = c(1 - \tilde{U})$. Substituting this relation into (\ref{eq:OptTempEqn}), we obtain
\[
\frac{\partial\mathcal{F}_\infty^\gamma}{\partial\mu} = \frac{1}{2\mu c} \left(\frac{c\mu}{\gamma} - \sqrt{(c\mu+c+1)^2 - 4c}\right)\tilde{U}' = 0,
\]
and since $U' > 0$, an optimal $\mu^\ast$ occurs if and only if (\ref{eq:OptTempEqn}) holds. The rest of the proof proceeds as in the $c < 1$ case.
\end{proof}

\begin{proposition}[Monotonicity in the Bayes Free Energy]
\label{prop:Monotonicity}
The limiting Bayes free energy $\mathcal{F}_\infty^\gamma$ at $\lambda = \lambda^\ast$ decreases monotonically in $c \in (0,\infty)$.
\end{proposition}
\begin{proof}
First, we treat the $c < 1$ case. Using the closed form expression for $D(\mu,c)$ in Lemma \ref{lem:Deter}, 
\[
\mathcal{F}_\infty^\gamma = \frac{1}{2\gamma_0}(1 - c + \mu T) + \frac12 \log(2\pi\gamma_0) - \frac{c}{2}\log \mu
+ \frac12 \left[ \log\left(1 + \frac{T}{c}\right) - \frac{T}{c+T} - c\log\left(\frac{T}{c}\right) \right].
\]
Note that, at the optimal $\mu^\ast$, $\frac{\dd}{\dd c} \mathcal{F}_\infty^\gamma = \frac{\partial}{\partial c} \mathcal{F}_\infty^\gamma + \frac{\partial}{\partial \mu} \mathcal{F}_\infty^\gamma \cdot \frac{\partial \mu^\ast}{\partial c} = \frac{\partial}{\partial c} \mathcal{F}_\infty^\gamma$. Therefore,
\[
2 \frac{\dd \mathcal{F}_\infty^\gamma}{\dd c} = \left(\frac{T}{(T+c)^2} + \frac{\mu}{\gamma_0} - \frac{c}{T}\right)\frac{\partial T}{\partial c} + 1 - \frac1{\gamma_0} - \frac{T^2}{c(T+c)^2} +\log\left(\frac{c}{\mu T}\right).
\]
Differentiating $T$ in $c$, we find that
\begin{align*}
\frac{\partial T}{\partial c} &= \frac{1-\mu}{2\mu}+\frac{1}{2\mu}\left(\frac{(c\mu+c+1)(\mu+1)-2}{\sqrt{(c\mu+c+1)^{2}-4c}}\right)\\
&=\frac{(c\mu+c+1)(\mu+1)-(\mu-1)\sqrt{(c\mu+c+1)^{2}-4c}-2}{2\mu\sqrt{(c\mu+c+1)^{2}-4c}}\\
&=\frac{2c-(\mu-1)T}{\sqrt{(c\mu+c+1)^{2}-4c}}.
\end{align*}
Since $c\mu^\ast / \gamma_0 = \sqrt{(c\mu^\ast+c+1)^2 - 4c}$, at the optimal $\mu^\ast$,
\[
\frac{\partial T}{\partial c} = \gamma_{0}\cdot\frac{2c-\mu T+T}{c\mu}.
\]
Note that for any $c > 0$,
\[
\mu T = \frac{c-1-c\mu+\sqrt{(c\mu+c+1)^{2}-4c}}{2}<\frac{c-1-c\mu+c\mu+c+1}{2}=c.
\]
Recalling that $\log x < x - 1$ for any $x > 1$, $\log(c/(\mu T)) < c/(\mu T) - 1$. Therefore, $2 \frac{\dd \mathcal{F}_\infty^\gamma}{\dd c} < M$, where
\[
M = \left(\frac{T}{(T+c)^2} + \frac{\mu}{\gamma_0} - \frac{c}{T}\right)\frac{\gamma_0(2c-\mu T+T)}{c\mu} - \frac{1}{\gamma_0} - \frac{T^2}{c(T+c)^2} + \frac{c}{\mu T}.
\]
Since $T = (c-1-c\mu^\ast + c\mu^\ast / \gamma_0)/(2\mu^\ast)$ at the optimal $\mu^\ast$, after several calculations, we find that
\[
M = Q(\mu^\ast,c,\gamma_0) \frac{c^2(1-\gamma_0^2)(\mu^\ast)^2 - 2c\mu^\ast (c+1)\gamma_0^2 - (c-1)^2 \gamma_0^2}{2c\gamma_0\mu^\ast(c\gamma_0 \mu^\ast - c \gamma_0 + c\mu^\ast + \gamma_0)(c\gamma_0 \mu^\ast + c\gamma_0 - c\mu^\ast - \gamma_0)},
\]
where
\[
Q(\mu,c,\gamma) = (2(c-1)\gamma+c\mu)(c\mu^{2}+2c\mu+c+\mu-1)\gamma^{2}+c\mu(c\mu+c+\mu-1)\gamma^{2} -(\mu+1)(c\mu+(c-1)\gamma)^{2}.
\]
In particular, by (\ref{eq:OptTempQuad}), at $\mu = \mu^\ast$, $M = 0$, and hence, $\frac{\dd}{\dd c}\mathcal{F}_\infty^\gamma < 0$.

Next, we treat the $c > 1$ case. Using the closed form expression in Lemma \ref{lem:Deter}, 
\[
\mathcal{F}_\infty^\gamma = \frac{\mu}{2\gamma_0}\tilde{T} - \frac{1}{2}\log\left(\frac{\mu}{2\pi\gamma_0}\right) + \frac12 c \log(c + \tilde{T}) - \frac12 c \log c - \frac12 \frac{c \tilde{T}}{c+\tilde{T}} - \frac12 \log \tilde{T}. 
\]
Differentiating in $c$ at the optimal $\mu^\ast$,
\begin{align*}
2\frac{\dd \mathcal{F}_\infty^\gamma}{\dd c} = 2\frac{\partial \mathcal{F}_\infty^\gamma}{\partial c} &= \left(\frac{\mu}{\gamma_0} - \frac{1}{\tilde{T}} + \frac{c\tilde{T}}{(c+\tilde{T})^2}\right) \frac{\partial \tilde{T}}{\partial c} \\
&+ \log(c+\tilde{T})-\log c - 1 + \frac{c\tilde{T}}{(c+\tilde{T})^2} + \frac{c^2}{(c+\tilde{T})^2} - \frac{\tilde{T}^2}{(c+\tilde{T})^2}.
\end{align*}
Differentiating $\tilde{T}$ in $c$, we find that
\begin{align*}
\frac{\partial\tilde{T}}{\partial c}	&=\frac{-1-\mu+\frac{(c\mu+c+1)(\mu+1)-2}{\sqrt{(c\mu+c+1)^{2}-4c}}}{2\mu}\\
&=\frac{-\left(\mu+1\right)\sqrt{(c\mu+c+1)^{2}-4c}+(c\mu+c+1)(\mu+1)-2}{2\mu\sqrt{(c\mu+c+1)^{2}-4c}}\\
&=\frac{1}{\sqrt{(c\mu+c+1)^{2}-4c}}\cdot\left[(\mu+1)\frac{c\mu+c+1-\sqrt{(c\mu+c+1)^{2}-4c}}{2\mu}-\frac{1}{\mu}\right]\\
&=\frac{1}{\sqrt{(c\mu+c+1)^{2}-4c}}\cdot\left[(\mu+1)\left(\frac{1}{\mu}-\tilde{T}\right)-\frac{1}{\mu}\right]\\
&=\frac{1-\mu \tilde{T} - \tilde{T}}{\sqrt{(c\mu+c+1)^{2}-4c}}.
\end{align*}
Since $c\mu^\ast / \gamma_0 = \sqrt{(c\mu^\ast+c+1)^2 - 4c}$, it follows that
\[
\frac{\partial\tilde{T}}{\partial c} = \gamma\cdot\frac{1-\mu T-T}{c\mu}.
\]
Note that for any $c > 0$, $\tilde{T} < c$, and so $\log(1 + \tilde{T}/c) < \tilde{T} / c$. Therefore, $2 \frac{\dd \mathcal{F}_\infty^\gamma}{\dd c} < M$ where
\[
M = \frac{\gamma_0}{c\mu}\left(\frac{\mu}{\gamma_0} - \frac{1}{\tilde{T}} + \frac{c\tilde{T}}{(c+\tilde{T})^2}\right)
+ \frac{\tilde{T}}{c} - 1 + \frac{c\tilde{T}}{(c+\tilde{T})^2} + \frac{c^2}{(c+\tilde{T})^2} - \frac{\tilde{T}^2}{(c+\tilde{T})^2}.
\]
Since $\tilde{T} = (1-c - c\mu^\ast + c\mu^\ast/\gamma_0) / (2\mu^\ast)$ at the optimal $\mu^\ast$, after several calculations, we find that
\[
M = -Q(\mu^{\ast},c,\gamma_0) \frac{c^2 (1-\gamma_0^2) (\mu^{\ast})^2 - 2 c \mu^\ast (c+1)\gamma_0^2 - (c-1)^2\gamma_0^2}{2c\gamma_0\mu^{\ast} (c\gamma_0 \mu^{\ast} - c\gamma_0 + c \mu^\ast + \gamma_0)(c\gamma_0 \mu^\ast +c \gamma_0 - c\mu^\ast - \gamma_0)},
\]
where
\begin{multline*}
Q(\mu,c,\gamma)	=\mu(c\mu+\gamma)^{2}+2c(\mu+1)^{2}(c-1)\gamma^{3}+2(c-1)(\mu-1)\gamma^{3}
	\\ -2c^{2}\gamma^{2}\mu(\mu+1)-\mu c^{2}\gamma^{2}(\mu+1)^{2}-2c\gamma^{2}\mu(\mu-1).
\end{multline*}
In particular, since the numerator for $M$ is always zero, it follows that $\frac{\dd \mathcal{F}_\infty^\gamma}{\dd c} < 0$. 
\end{proof}

Theorem \ref{thm:Main} follows immediately from Propositions \ref{prop:EntropyLimit}, \ref{prop:OptTempBayes}, \ref{prop:OptRegBayes}, and \ref{prop:Monotonicity}. 

\begin{proof}[Proof of Proposition \ref{prop:DD}]
Under the stated hypotheses, let $\delta(\lambda,\gamma) = c_2(\lambda,\gamma) / c_1(\gamma)$ and $\bar{E}(c) = E(c) + c_3(\gamma) / c_1(\gamma)$. Then 
\begin{align*}
|\mathcal{L} / c_1 - \bar{E}| &\leq |\mathbb{E}\|\bar{f}(\boldsymbol{x}) - \boldsymbol{y}\|^2 - E| + \delta(\lambda(\gamma),\gamma) \mathbb{E}\tr(\Sigma(\boldsymbol{x})) \\
&\leq |\mathbb{E}\|\bar{f}(\boldsymbol{x}) - \boldsymbol{y}\|^2 - E| + \delta(\lambda(\gamma),\gamma) m \mathbb{E}k(x, x).
\end{align*}
For an arbitrary $\epsilon > 0$, let $N$ be sufficiently large so that for any $n > N$ and $d = d(n)$, $|\mathbb{E}\|\bar{f}(\boldsymbol{x}) - \boldsymbol{y}\|^2 - E| \leq \epsilon / 2$. Similarly, let $\gamma_0$ be sufficiently small so that for any $0 < \gamma < \gamma_0$, $\delta(\lambda(\gamma),\gamma) < \epsilon / (2 m \mathbb{E}k(x, x))$. Then $|\mathcal{L} / c_1 - \bar{E}| < \epsilon$, and the result follows. 
\end{proof}

\fi
\end{document}